\documentclass[11pt, a4paper]{article}

\usepackage{amsmath}
\usepackage{amssymb}
\usepackage{mathtools}
\usepackage{graphicx}
\usepackage{hyperref}
\usepackage{geometry}
\usepackage{algorithm}
\usepackage{algpseudocode}
\usepackage{booktabs}
\usepackage{caption}
\usepackage{subcaption}
\usepackage{array}
\usepackage{physics}
\usepackage{enumitem}
\usepackage[svgnames]{xcolor}
\usepackage{wrapfig}
\usepackage{placeins}

\usepackage{newtxtext}
\usepackage{newtxmath}
\usepackage{bm}

\usepackage{amsthm}
\usepackage{thmtools}
\usepackage{thm-restate}

\theoremstyle{plain}
\newtheorem{theorem}{Theorem}[section]

\theoremstyle{definition}
\newtheorem{definition}[theorem]{Definition}

\theoremstyle{remark}

\DeclareMathOperator{\hsic}{HSIC}
\DeclareMathOperator{\dcov}{dCov}

\usepackage{siunitx}

\hypersetup{
  colorlinks = true,
  allcolors = DarkRed,
}

\usepackage[backend=biber,sorting=none]{biblatex}
\DeclareFieldFormat{urldate}{}

\DeclareSourcemap{
  \maps[datatype=bibtex]{
    \map{
      \step[fieldset=primaryclass, null]
    }
  }
}

\makeatletter
\def\@fnsymbol#1{\ensuremath{\ifcase#1\or a\or b\or c\or d\or e\or f\or g \or h \else\@ctrerr\fi}}
\makeatother



\title{From Layers to Networks: Comparing Neural Representations via Diffusion Geometry}

\author{%
  Atharva Khandait\thanks{Department of Mathematical Sciences, Chalmers University of Technology and the University of Gothenburg, SE-412 96 Gothenburg, Sweden.\\ Emails: khandait@student.chalmers.se, gerken@chalmers.se} \and
  Jan E.\ Gerken\footnotemark[1]
}

\date{}

\begin{document}

\maketitle

\begin{abstract}
  Diffusion geometry is a manifold learning framework that uses random walks defined by Markov transition matrices to characterize the geometry of a dataset at multiple scales. We use diffusion geometry for neural representations, incorporating tools from multi-view learning into this field for the first time. Our key technical observation is that a broad class of similarity measures based on representational similarity matrices (RSMs) admits a closed-form equivalent formulation in terms of row-stochastic Markov matrices, opening the door to manipulations from diffusion geometry. As a first application, we develop multi-scale variants of Centered Kernel Alignment and Distance Correlation, which utilise the $t$\textsuperscript{th} power of the underlying transition matrix to probe the data geometry at adjustable diffusion scales. Going further, we introduce variants of these measures which fuse the Markov matrices of several layers via alternating diffusion into a single operator that captures the network's joint sample geometry, allowing similarity to be computed across multiple layers and shifting the comparison from layer-to-layer to network-to-network. We perform extensive numerical experiments, evaluating our measures on the Representational Similarity (ReSi) benchmark comprising 14 architectures trained on 7 datasets across three different domains. Our methods achieve SoTA results in accuracy and output correlation for both language and vision tasks across different models. We furthermore show SoTA performance on an additional benchmark evaluating on out-of-distribution data.

  
\end{abstract}

\section{Introduction}

Similarity measures of neural representations have become a central tool for analyzing and comparing trained networks~\cite{kornblith2019similarity, kriegeskorte2008rsa, klabunde2025resi}, as well as for network optimization, for instance in knowledge distillation~\cite{tung2019similarity}. Each layer of a neural network maps the input data into a representation space, inducing a discrete sample of an underlying data manifold. Similarity measures quantify, layer by layer, how the geometry of this sample compares across networks, training runs, or architectural choices. The most widely used families of measures, including Centered Kernel Alignment (CKA)~\cite{kornblith2019similarity} and Distance Correlation (DistCorr)~\cite{szekely2007distance}, are built on representational similarity matrices (RSMs)~\cite{kriegeskorte2008rsa}, which abstract away the ambient dimension by recording the pairwise similarities of activations within a single layer.

A complementary line of work, diffusion geometry~\cite{coifman2006diffusion}, characterizes the intrinsic geometry of a dataset through a Markov random walk. The transition matrix $\bm{P}$ encodes local connectivity, while its powers $\bm{P}^{t}$ aggregate length-$t$ transitions and reveal the geometry of the data at successively coarser scales. Diffusion-geometric tools have been particularly successful in multi-view learning~\cite{xu2013survey}, where an underlying phenomenon must be extracted from several partially redundant observations. In particular, alternating diffusion (AD)~\cite{lederman2018alternating} fuses the Markov operators of several views by forming their product: under a conditional independence assumption, the resulting operator describes a random walk on the shared latent space, while view-specific nuisance variability is suppressed.

The multi-view learning setting is a natural fit for neural representations, since the successive layers of a network can be understood as multiple, partially redundant views of the same input. Despite this, diffusion-geometric and multi-view learning ideas have not been brought to bear on similarity measures of neural representations. We close this gap. Our key technical observation is that a broad class of RSM-based similarity measures admits a closed-form, equivalent reformulation in terms of row-stochastic Markov matrices (Theorem~\ref{thm:markovablemeasures}): any RSM-based measure that acts on centered RSMs and is invariant under separate positive rescalings of its arguments can be rewritten in terms of an explicit Markov matrix obtained from the RSM by an affine shift-and-rescale. This reformulation opens the door to manipulations from diffusion geometry in the design of similarity measures.

We use this lens to derive two new families of measures, instantiated for both CKA and DistCorr. First, by utilising the $t$-th power of the underlying Markov matrix, we obtain multi-scale variants MS-CKA and MS-DistCorr that probe the data geometry at adjustable diffusion scales. Second, by fusing the Markov matrices of several layers via alternating diffusion into a single operator, we obtain AD-CKA and AD-DistCorr. The latter shifts the comparison from layer-to-layer to network-to-network: the fused operator captures geometric features characteristic of an entire network rather than of any single layer, so the resulting measures depend on multiple layers of each network simultaneously.


Our main contributions are as follows:
\begin{itemize}
  \item We propose a framework, inspired by multi-view learning, that reformulates a broad class of representational similarity measures in terms of Markov matrices and thereby makes them amenable to tools from diffusion geometry, including generalizations that take several layers into account simultaneously.
  \item We prove (Theorem~\ref{thm:markovablemeasures}) that any centered, scale-invariant RSM-based similarity measure admits a closed-form Markov reformulation, and verify that both CKA and DistCorr fall within this class. Building on this, we introduce multi-scale variants MS-CKA and MS-DistCorr and alternating-diffusion variants AD-CKA and AD-DistCorr, the latter of which compare networks at the level of multiple layers rather than individual layer pairs.
  \item We evaluate the proposed measures on the  Representational Similarity benchmark~\cite{klabunde2025resi} and on out-of-distribution accuracy correlation~\cite{ding2021grounding}, and find that our measures set a new SoTA in several evaluation settings across different architectures for language and vision tasks.
  
  
\end{itemize}


\section{Related Works}

\paragraph{Representational similarity measures.}
Quantifying similarity between neural representations has produced a sizeable family of measures, including representational similarity analysis~\cite{kriegeskorte2008rsa}, CKA~\cite{kornblith2019similarity}, CCA-based variants such as SVCCA~\cite{raghu2017svcca} and PWCCA~\cite{morcos2018insights}, and distance correlation~\cite{zhen2022versatile}; \cite{klabunde2025similarity} provides a recent taxonomic survey. With few exceptions, these measures compare a single pair of layers, whereas our construction aggregates information across layers and extends two of the above (CKA and DistCorr) to this multi-layer setting. Existing network-to-network comparisons often reduce the problem to a best-match layer heuristic~\cite{huh2024platonic}, computing similarity for all cross-network layer pairs and then retaining the maximum score, while recent work~\cite{shah2026alignment} has proposed global alignment via optimal transport, which addresses a related but not identical alignment problem.

\paragraph{Benchmarking similarity measures.}
Several benchmarks evaluate such measures along complementary axes: \cite{ding2021grounding} use sensitivity/specificity tests, ReSi~\cite{klabunde2025resi} provides a comprehensive suite covering multiple architectures and tasks, and \cite{bo2024evaluating} compare measures via functional correspondence with behavioural metrics. We evaluate our method on ReSi and benchmarks from~\cite{ding2021grounding}.

\paragraph{Diffusion geometry and alternating diffusion.}
Diffusion maps~\cite{coifman2006diffusion,coifman2005geometric} and its fractional variant~\cite{antil2021fractional} construct a multi-scale geometry of a dataset from the eigendecomposition of a Markov operator, and form the mathematical basis of our Markov-matrix formulation. Alternating diffusion (AD)~\cite{lederman2015alternating,lederman2018alternating} extends this idea to two views by composing their Markov operators to recover the geometry of common latent variables; \cite{talmon2019latent} sharpens the underlying analysis, and \cite{katz2019alternating} generalises AD to more than two views, the variant most directly relevant to fusing $L$ layers. AD has been applied to multimodal sensor fusion in audio-visual settings~\cite{dov2018audio} and to fuse fMRI connectivity matrices for IQ prediction~\cite{xiao2019alternating}. AD originates in the multi-view learning literature, surveyed in~\cite{xu2013survey,li2019survey,sun2013survey}, where it sits among non-linear fusion methods as opposed to alignment-based approaches such as CCA.

\paragraph{Diffusion geometry on neural representations.}
Diffusion-geometric tools have also been applied to deep networks. M-PHATE~\cite{gigante2019visualizing}, building on PHATE~\cite{moon2019visualizing}, uses diffusion-based embeddings to visualise how hidden representations evolve during training; \cite{abel2024exploring} define a diffusion-geometric distance between hidden-layer representations to construct a manifold of networks; and \cite{fasina2023neural} learn Fisher information metrics from point clouds. None of these works yields a similarity measure between two networks, the gap addressed by our method.


\section{Background}
In this section, we lay the groundwork for our method by introducing two families of concepts. We first review representational similarity matrices (RSMs) and the scalar similarity measures derived from them, focusing on Centered Kernel Alignment (CKA) and Distance Correlation (DistCorr). We then introduce the multi-view learning framework and Alternating Diffusion (AD), a principled method for extracting shared geometric structure from multiple Markov matrices, which forms the mathematical backbone of our approach.

\subsection{Representational similarity measures}
\label{subsec:rep_sim_m}

A neural network $f = f^{(L)} \circ f^{(L-1)} \circ \cdots \circ f^{(1)}$ applied to $N$ inputs $\{\bm{x}_i\}_{i=1}^{N}$ with design matrix $\bm{X}=(\bm{x}_{1}, \bm{x}_{2},\dots,\bm{x}_{N})^{\top}$ yields, at each layer $l$, a \textit{representation} $\bm{R} := \bm{R}^{(l)} = \bigl(f^{(l)} \circ \cdots \circ f^{(1)}\bigr)(\bm{X}) \in \mathbb{R}^{N \times D}$, where the rows $\bm{R}_{i,:} \in \mathbb{R}^{D}$ are the \textit{instance representations}~\cite{klabunde2025resi}. A \textit{representational similarity measure} (or \textit{similarity measure} for brevity) is a function $m : \mathbb{R}^{N \times D} \times \mathbb{R}^{N \times D'} \to \mathbb{R}$ that assigns a scalar score $m(\bm{R}_1, \bm{R}_2)$ to a pair of representations, which may come from different layers of the same or different networks~\cite{klabunde2025resi}. 

Representational Similarity Matrices (RSMs)\footnote{We will use the abbreviation RSM only to refer to representational similarity matrices, not to representational similarity measures.} capture the pairwise similarity structure among a fixed set of inputs, thereby abstracting away the ambient representation space. The RSM $\bm{S} \in \mathbb{R}^{N \times N}$ of a representation $\bm{R}$ is defined element-wise as $S_{ij} = s\!\left(\bm{R}_{i,:}, \bm{R}_{j,:}\right)$, where $s : \mathbb{R}^D \times \mathbb{R}^D \to \mathbb{R}$ is a chosen instance-wise similarity function. Common choices include Pearson correlation~\cite{kriegeskorte2008rsa} and kernel functions~\cite{kornblith2019similarity}. Hence, the RSM is the Gram matrix of $s$ over the samples in $\bm{R}$.

Among the several families of similarity measures~\cite{klabunde2025resi}, we focus on \textit{RSM-based} measures, for which $m(\bm{R}_1, \bm{R}_2) = m_{\text{RSM}}(\bm{S}_1, \bm{S}_2)$ for some function $m_{\text{RSM}}$ acting on the RSMs $\bm{S}_1,\bm{S}_2\in\mathbb{R}^{N\times N}$ of the two representations. We restrict our attention to two specific instances, CKA and DistCorr, both for empirical reasons and because they admit natural Markov matrix analogues (see Section~\ref{subsec:sim_markov}).

\paragraph{Centered Kernel Alignment (CKA).}
CKA is among the most widely used similarity measures for neural representations. It uses kernel functions as instance-wise similarity functions $s$ to construct the RSMs~\cite{kornblith2019similarity} and quantifies their alignment via the Hilbert-Schmidt Independence Criterion (HSIC)~\cite{gretton2005hsic} according to
\begin{align}
  m_{\text{CKA}}(\bm{S}_1, \bm{S}_2) = \frac{\hsic(\bm{S}_1, \bm{S}_2)}{\sqrt{\hsic(\bm{S}_1, \bm{S}_1)\,\hsic(\bm{S}_2, \bm{S}_2)}}\,,\label{eq:1}
\end{align}
where the empirical estimator of the HSIC is given by
\begin{align}
  \hsic(\bm{S}_1, \bm{S}_2) = \frac{1}{(N-1)^2}\text{tr}(\bm{S}_1\bm{H}\bm{S}_2\bm{H})
  \quad\text{with}\quad
  \bm{H}=\bm{I} - \frac{1}{N}\bm{1}\bm{1}^\top\,,\label{eq:2}
\end{align}
where $\bm{H}$ is the column-wise centering matrix. Here and in the following, $\bm{I}\in\mathbb{R}^{N\times N}$ denotes the unit matrix and $\bm{1}\in\mathbb{R}^{N}$ the vector whose components are all one.

Two common choices of kernel are the linear kernel $s(\bm{r}_i, \bm{r}_j) = \bm{r}_i^\top \bm{r}_j$ and the RBF kernel $s(\bm{r}_i, \bm{r}_j) = \exp\!\left(-\|\bm{r}_i - \bm{r}_j\|_2^2 / (2\sigma^2)\right)$~\cite{kornblith2019similarity}, giving rise to \emph{Linear CKA} and \emph{RBF CKA}, respectively.

\paragraph{Distance Correlation (DistCorr).}
Distance Correlation~\cite{szekely2007distance} originates from nonparametric tests of statistical independence between random variables, but serves equally well as a representation similarity measure. Assuming the RSMs are mean-centered in both rows and columns, the squared sample distance covariance is defined by $\dcov^2(\bm{S}_1, \bm{S}_2) = \frac{1}{N^{2}}\sum_{i,j=1}^{N}(\bm{S}_1)_{ij}(\bm{S}_2)_{ij}$, yielding the similartiy score
\begin{align}
  m_{\text{DistCorr}}(\bm{S}_1, \bm{S}_2) = \frac{\dcov^2(\bm{S}_1, \bm{S}_2)}{\sqrt{\dcov^2(\bm{S}_1, \bm{S}_1)\,\dcov^2(\bm{S}_2, \bm{S}_2)}}\,.
  \label{eq:6}
\end{align}

\subsection{Multi-view learning and diffusion geometry}\label{subsec:ad}

Multi-view learning addresses settings in which a single underlying phenomenon is observed through several modalities, sensors, or feature extractors~\cite{xu2013survey}. Each \textit{view} typically combines signal that is shared across modalities with view-specific nuisance variability, and the aim is to integrate the views in a way that preserves the shared structure while attenuating the view-specific factors. We focus on \textit{alternating diffusion} (AD), a nonlinear manifold-learning technique introduced by~\cite{lederman2018alternating} that fuses two or more views by alternately running a diffusion process on each.

\paragraph{Markov matrices from a single view.}
Diffusion-based methods describe the geometry of a dataset through a Markov chain whose transitions concentrate on points that lie close to each other in the input space~\cite{coifman2006diffusion}. Given $N$ samples in a single view $\bm{V} \in \mathbb{R}^{N \times D}$, one builds an affinity matrix $\bm{W} \in \mathbb{R}^{N \times N}$ from a local kernel, for instance the Gaussian kernel $W_{ij} = \exp\!\left(-\|\bm{V}_{i,:} - \bm{V}_{j,:}\|_2^2 / \varepsilon\right)$, and row-normalizes it into a row-stochastic \textit{Markov matrix} $\bm{P} \in \mathbb{R}^{N \times N}$ with $P_{ij} = W_{ij} / \sum_{k} W_{ik}$. The entry $P_{ij}$ is the one-step transition probability of a random walk on the samples, and the $t$-th matrix power $\bm{P}^t$ encodes $t$-step transition probabilities, providing a multiscale description of the data manifold at progressively coarser resolutions~\cite{coifman2006diffusion}.

\paragraph{Alternating diffusion across views.}
AD assumes that two views observe a shared latent phenomenon corrupted by view-specific nuisance factors. Following~\cite{lederman2018alternating}, let $X$ be a hidden random variable encoding the shared content and $Y, Z$ be view-specific hidden random variables. The two views are generated through possibly nonlinear observation maps $\bm{V}_1 = g(X, Y)$ and $\bm{V}_2 = h(X, Z)$, and we observe $N$ aligned samples corresponding to the same realizations of $X$. The crucial modelling assumption is that $Y$ and $Z$ are \emph{conditionally independent given $X$}, formalizing the idea that they are independent sources of view-specific noise. Letting $\bm{P}_1, \bm{P}_2 \in \mathbb{R}^{N \times N}$ be the per-view Markov matrices constructed as above, the \textit{alternating-diffusion operator} is the row-stochastic product $\bm{P}_{\text{AD}} = \bm{P}_2\,\bm{P}_1$. A transition from sample $i$ to sample $j$ has appreciable probability under $\bm{P}_{\text{AD}}$ only if there exists an intermediate sample $k$ that is close to $i$ in the first view \emph{and} close to $j$ in the second view, so pairs that look similar only because of view-specific variability are penalized while pairs sharing the same value of $X$ are reinforced. The two orderings $\bm{P}_2\,\bm{P}_1$ and $\bm{P}_1\,\bm{P}_2$ share the same non-zero eigenvalues and induce equivalent diffusion geometries on the common variable, and the construction extends to $M \geq 2$ views by taking the product of all per-view Markov matrices~\cite{lederman2018alternating}.

The alternation isolates the shared geometry in a precise sense: under the conditional-independence assumption, the action of $\bm{P}_{\text{AD}}$ on observables coincides, after marginalizing over $Y$ and $Z$, with that of an effective diffusion operator on the latent space of $X$ alone (\cite{lederman2018alternating}, Thm.~3 and Cor.~4). Consequently, the diffusion geometry induced by $\bm{P}_{\text{AD}}$ is precisely that of the unobserved common variable $X$, and diffusion distances computed from $\bm{P}_{\text{AD}}$ approximate, up to a constant, true diffusion distances on the manifold of $X$ (\cite{lederman2018alternating}, Thm.~5). The guarantee holds for arbitrary nonlinear $g$ and $h$, in contrast to linear approaches such as canonical correlation analysis.


\section{Diffusion geometry for representational similarity matrices}
\label{sec:method}
In order to combine the concepts introduced above, we will first rewrite a large class of similarity measures in terms of Markov matrices. These are then amenable to techniques from diffusion geometry.

\subsection{Similarity measures in terms of Markov matrices}
\label{subsec:sim_markov}
In this subsection, we connect representational similarity measures to Markov matrices. The first step is to recognize that a broad class of RSM-based similarity measures $m(\bm{R}_1, \bm{R}_2) = m_{\text{RSM}}(\bm{S}_1, \bm{S}_2)$ can be reformulated equivalently as $m_{\bm{P}}(\bm{P}_1, \bm{P}_2)$, where $\bm{P}_1, \bm{P}_2 \in \mathbb{R}^{N\times N}$ are row-stochastic matrices obtained from $\bm{S}_1$ and $\bm{S}_2$ by an explicit shift-and-rescale.

Informally, the main result of this subsection is that any RSM-based similarity measure that acts on centered RSMs and is invariant under separate positive rescalings of its arguments admits such a Markov reformulation via an affine embedding of the centered RSM into the space of row-stochastic matrices. This reformulation is the bridge that allows us to bring the multi-view learning machinery of Section~\ref{subsec:ad} to bear on representational similarity in the next subsections.

\begin{restatable}[]{theorem}{markovablemeasures}
  \label{thm:markovablemeasures}
  Let $m_{\text{RSM}}$ be an RSM-based similarity measure that can be written as
  \begin{align}
    m_{\text{RSM}}(\bm{S}_1,\bm{S}_2) = \psi(\bm{C_{q}}(\bm{S}_1), \bm{C_{q}}(\bm{S}_2))\,,
  \end{align}
  where  $\psi$ is invariant to separate positive rescalings,
  \begin{align}
    \psi(a\bm{A}, b\bm{B}) = \psi(\bm{A}, \bm{B}) \qquad \forall\ a, b > 0\,,
  \end{align}
  $\bm{q}\in\mathbb{R}^{N}$ is a probability vector with positive components and $\bm{C_{q}} : \mathbb{R}^{N \times N} \to \mathbb{R}^{N \times N}$ is a $\bm{q}$-weighted centering operator satisfying
  \begin{align}
    \bm{C_{q}}^2 = \bm{C_{q}}, \qquad \bm{C_{q}}(\bm{M})\mathbf{1} = \bm{0}, \qquad \bm{C_{q}}(\mathbf{1}\bm{q}^\top) = \bm{0}\label{eq:4}
  \end{align}
  for any $\bm{M}\in\mathbb{R}^{N\times N}$\,.

  Let furthermore $\bm{P}(\bm{S}):\mathbb{R}^{N\times N}\rightarrow\mathbb{R}^{N\times N}$ be defined by
  \begin{align}
    \bm{P}(\bm{S}) = \mathbf{1}\bm{q}^\top + \alpha(\bm{S}) \bm{C_{q}}(\bm{S})\,,
    \quad\text{where}\quad
    \alpha(\bm{S})=\min_{(\bm{C_{q}}(\bm{S}))_{ij}\neq 0}\frac{q_{j}}{|(\bm{C_{q}}(\bm{S}))_{ij}|}\,.
  \end{align}

  Then, $\bm{P}(\bm{S}_1), \bm{P}(\bm{S}_2)$ are Markov matrices and $m_{\text{RSM}}$ can be written as
  \begin{align}
    m_{\text{RSM}}(\bm{S}_1,\bm{S}_2) = \psi(\bm{C_{q}}(\bm{P}(\bm{S}_1)), \bm{C_{q}}(\bm{P}(\bm{S}_2)))\,.
  \end{align}
\end{restatable}

\emph{Proof idea:} Adding the rank-one Markov matrix $\mathbf{1}\bm{q}^\top$ to a sufficiently rescaled centered RSM produces a row-stochastic matrix: the rows of $\mathbf{1}\bm{q}^\top$ sum to one, $\bm{C_{q}}(\bm{S})$ has zero row sums by~(\ref{eq:4}), and $\alpha(\bm{S})$ is the largest positive scalar that keeps every entry of $\bm{P}(\bm{S})$ nonnegative. Centering this Markov matrix removes the additive shift, again by~(\ref{eq:4}), and leaves the centered RSM up to the positive factor $\alpha(\bm{S})$. The scale invariance of $\psi$ then absorbs this factor, recovering $m_{\text{RSM}}(\bm{S}_1, \bm{S}_2)$.

For the full proof, see Appendix~\ref{sec:theorem1proof}

The CKA and DistCorr similarity measures we will analyze in the following both rely on the same centering operator, namely double-sided $\bm{q}$-weighted centering. To link them to Theorem~\ref{thm:markovablemeasures}, we now verify that this operator satisfies the conditions stated there.

\begin{restatable}[]{lemma}{centeringoplemma}
  \label{lem:centeringop}
  The double-sided $\bm{q}$-weighted centering operator $\bm{C_{q}} : \mathbb{R}^{N \times N} \to \mathbb{R}^{N \times N}$ defined by
  \begin{align}
    \bm{C_{q}}(\bm{M})=(\bm{I}-\bm{1q}^{\top})\bm{M}(\bm{I}-\bm{1q}^{\top})
  \end{align}
  for a probability vector $\bm{q}\in\mathbb{R}^{N}$ satisfies the conditions of Theorem~\ref{thm:markovablemeasures} in~(\ref{eq:4}). For $\bm{q}=\frac{1}{N}\bm{1}$, we have $\bm{C_{\frac{1}{N}\bm{1}}}(\bm{M})=\bm{H}\bm{MH}$, with the column-wise centering matrix $\bm{H}$ defined in~(\ref{eq:2}).
\end{restatable}
\emph{Proof:} See Appendix~\ref{sec:theorem1proof}.

With Theorem~\ref{thm:markovablemeasures} and Lemma~\ref{lem:centeringop} in hand, the framework covers a broad family of RSM-based similarity measures. Linear and RBF CKA, distance correlation, and the Eigenspace Overlap Score all satisfy its assumptions, and several further measures fall in scope after minor preprocessing. In the rest of this subsection, we work out the Markov reformulation explicitly for CKA and DistCorr, the two measures used in our experiments.

\begin{restatable}[]{corollary}{ckacorollary}
  \label{cor:cka}
  The centered kernel alignment measure defined in~(\ref{eq:1}) can be written as
  \begin{align}
    m_{\text{CKA}}(\bm{S}_{1},\bm{S}_{2})=\psi(\bm{C_{q}}(\bm{P}(\bm{S}_1)), \bm{C_{q}}(\bm{P}(\bm{S}_2)))
  \end{align}
  with Markov matrices $\bm{P}(\bm{S}_1)$ and $\bm{P}(\bm{S}_2)$.
\end{restatable}
\emph{Proof:} By direct computation, see Appendix~\ref{sec:theorem1proof}.

Note that Theorem~\ref{thm:markovablemeasures} is constructive: given the RSMs $\bm{S}_1$ and $\bm{S}_2$, the Markov matrices $\bm{P}(\bm{S}_1)$ and $\bm{P}(\bm{S}_2)$ are obtained in closed form, so the corollary provides an explicit recipe for computing CKA through Markov matrices.

\begin{restatable}[]{corollary}{distcorrcorollary}
  \label{cor:distcorr}
  The distance correlation measure defined above~(\ref{eq:6}) can be written as
  \begin{align}
    m_{\text{DistCorr}}(\bm{S}_{1},\bm{S}_{2})=\psi(\bm{C_{q}}(\bm{P}(\bm{S}_1)), \bm{C_{q}}(\bm{P}(\bm{S}_2)))
  \end{align}
  with Markov matrices $\bm{P}(\bm{S}_1)$ and $\bm{P}(\bm{S}_2)$.
\end{restatable}
\emph{Proof:} By direct computation, see Appendix~\ref{sec:theorem1proof}.


\subsection{Multi-Scale Diffusion Similarity Measures}
\label{subsec:multiscale}
With CKA and DistCorr now expressed in terms of Markov matrices, we can manipulate these operators to probe the geometry of the data distribution beyond the direct pairwise affinities encoded in the original RSMs. The simplest such manipulation is to take powers of the Markov matrix.

Higher powers $\bm{P}(\bm{S})^t$ describe $t$-step transitions of the random walk induced by $\bm{P}(\bm{S})$: the entry $(\bm{P}(\bm{S})^t)_{ij}$ aggregates all length-$t$ paths from sample $i$ to sample $j$ and thereby probes the data geometry at scale $t$, the viewpoint underlying diffusion-based manifold learning~\cite{coifman2006diffusion}. Since $\bm{P}(\bm{S})^t$ remains row-stochastic, it is a valid input to the $\bm{q}$-weighted centering operator and the similarity functions $\psi$ of Theorem~\ref{thm:markovablemeasures}. In this respect, multi-scale measures are conceptually related to topology-based similarity measures~\cite{barannikov2022rtd,perez2021persistent,tsitsulin2020shape}, which probe higher-order neighborhood structure.

Concretely, using $\bm{P}(\bm{S})^t$ in Corollaries~\ref{cor:cka} and~\ref{cor:distcorr} yields multi-scale variants of CKA and DistCorr.

\begin{definition}[Multi-Scale CKA]
  \label{def:mscka}
  For RSMs $\bm{S}_1, \bm{S}_2 \in \mathbb{R}^{N\times N}$ and an integer $t\geq 1$, the \emph{Multi-Scale CKA} similarity measure is defined as
  \begin{align}
    m_{\text{MS-CKA}}^{(t)}(\bm{S}_1, \bm{S}_2)
    = \frac{\hsic(\bm{P}(\bm{S}_1)^{t}, \bm{P}(\bm{S}_2)^{t})}
           {\sqrt{\hsic(\bm{P}(\bm{S}_1)^{t}, \bm{P}(\bm{S}_1)^{t})\,\hsic(\bm{P}(\bm{S}_2)^{t}, \bm{P}(\bm{S}_2)^{t})}}\,,
  \end{align}
  where the outer $\bm{q}$-weighted centering of Theorem~\ref{thm:markovablemeasures} has been absorbed by the $\bm{H}$-centering already built into $\hsic$ (cf.~the proof of Corollary~\ref{cor:cka}). For the uniform choice $\bm{q}=\frac{1}{N}\bm{1}$ used by CKA, $\bm{C_{q}}(\bm{M})=\bm{H}\bm{M}\bm{H}$ and $\bm{P}(\bm{S})$ takes the closed form
  \begin{align}
    \bm{P}(\bm{S}) = \tfrac{1}{N}\bm{1}\bm{1}^{\top} + \alpha(\bm{S})\,\bm{H}\bm{S}\bm{H}\,,
    \qquad
    \alpha(\bm{S}) = \min_{(\bm{H}\bm{S}\bm{H})_{ij}\neq 0}\frac{1}{N\,|(\bm{H}\bm{S}\bm{H})_{ij}|}\,,
  \end{align}
  so that the RHS of $m_{\text{MS-CKA}}^{(t)}$ is directly computable from $\bm{S}_1$ and $\bm{S}_2$.
\end{definition}

\begin{definition}[Multi-Scale DistCorr]
  \label{def:msdistcorr}
  For RSMs $\bm{S}_1, \bm{S}_2 \in \mathbb{R}^{N\times N}$ and an integer $t\geq 1$, the \emph{Multi-Scale DistCorr} similarity measure is defined as
  \begin{align}
    m_{\text{MS-DistCorr}}^{(t)}(\bm{S}_1, \bm{S}_2)
    = \frac{\langle \bm{H}\bm{P}(\bm{S}_1)^{t}\bm{H},\, \bm{H}\bm{P}(\bm{S}_2)^{t}\bm{H}\rangle_F}
           {\|\bm{H}\bm{P}(\bm{S}_1)^{t}\bm{H}\|_F\,\|\bm{H}\bm{P}(\bm{S}_2)^{t}\bm{H}\|_F}\,,
  \end{align}
  where the outer $\bm{q}$-weighted centering of Theorem~\ref{thm:markovablemeasures} has been written explicitly via the centering matrix $\bm{H}$ (cf.~the proof of Corollary~\ref{cor:distcorr}). For the uniform choice $\bm{q}=\frac{1}{N}\bm{1}$, $\bm{C_{q}}(\bm{M})=\bm{H}\bm{M}\bm{H}$ and $\bm{P}(\bm{S})$ takes the closed form
  \begin{align}
    \bm{P}(\bm{S}) = \tfrac{1}{N}\bm{1}\bm{1}^{\top} + \alpha(\bm{S})\,\bm{H}\bm{S}\bm{H}\,,
    \qquad
    \alpha(\bm{S}) = \min_{(\bm{H}\bm{S}\bm{H})_{ij}\neq 0}\frac{1}{N\,|(\bm{H}\bm{S}\bm{H})_{ij}|}\,,
  \end{align}
  so that the RHS of $m_{\text{MS-DistCorr}}^{(t)}$ is directly computable from $\bm{S}_1$ and $\bm{S}_2$.
\end{definition}

For $t=1$, both definitions reduce to standard CKA and DistCorr by Corollaries~\ref{cor:cka} and~\ref{cor:distcorr}. Larger values of $t$ progressively shift the comparison from direct pairwise affinities to multi-scale neighborhood structure, providing a tunable means of probing the geometry of the underlying data manifold at different scales.

\subsection{Alternating-Diffusion Similarity Measures}
\label{subsec:multirep}
Section~\ref{subsec:multiscale} expanded the input space of CKA and DistCorr through powers of a single Markov matrix. The multi-view machinery from Section~\ref{subsec:ad} suggests a complementary direction: rather than processing each representation in isolation, alternating diffusion lets us fuse \emph{several} Markov matrices into a single operator that captures their shared geometry. For neural networks, the natural source of multiple aligned views is the layer stack: applying AD across all layers of a network yields a representation of the network as a whole, and comparing two networks then compares these joint operators rather than two isolated layers. This shifts the representational similarity from layer-to-layer to network-to-network and exposes structure that no single layer carries on its own.

\begin{definition}[Multi-representation similarity measure]
  \label{def:multirepsim}
  Let $\mathcal{A} = \{\bm{R}_i \in \mathbb{R}^{N \times D_i}\}_{i=1}^n$ denote a set of $n$ sample-wise aligned representations, i.e.\ each $\bm{R}_i$ is computed from the same $N$ inputs but may live in a different ambient dimension $D_i$. Write $\mathfrak{A}$ for the collection of all such finite sets, where both $n$ and the $D_i$ may vary, and let $\mathcal{Y} \subset \mathbb{R}^{N\times N}$ denote the set of $N \times N$ row-stochastic matrices. Given a fusion map $\mathscr{F} : \mathfrak{A} \to \mathcal{Y}$ and a Markov-matrix similarity measure $m_{\bm{P}}$ as in Theorem~\ref{thm:markovablemeasures}, the \emph{multi-representation} similarity measure $m_{\text{MR}}$ between two such sets $\mathcal{A}_1, \mathcal{A}_2$ is defined as
  \begin{align}
    m_{\text{MR}}(\mathcal{A}_1, \mathcal{A}_2) = m_{\bm{P}}(\mathscr{F}(\mathcal{A}_1), \mathscr{F}(\mathcal{A}_2))\,.
  \end{align}
\end{definition}

Intuitively, $\mathscr{F}$ collapses an entire collection of aligned representations into a single Markov matrix that summarizes their joint sample geometry, and $m_{\bm{P}}$ then plays its usual role of comparing two such matrices. We instantiate $\mathscr{F}$ via alternating diffusion across the representations in the set, lifted to the Markov reformulation of Theorem~\ref{thm:markovablemeasures}.

\begin{definition}[Alternating-diffusion fusion map]
  \label{def:adfusion}
  For $\mathcal{A} = \{\bm{R}_i\}_{i=1}^n \in \mathfrak{A}$, let $\bm{S}_i$ denote the RSM of $\bm{R}_i$ and $\bm{P}(\bm{S}_i)$ the Markov matrix associated with $\bm{S}_i$ via Theorem~\ref{thm:markovablemeasures}. The \emph{alternating-diffusion fusion map} $\mathscr{F}_{\text{AD}} : \mathfrak{A} \to \mathcal{Y}$ is defined as
  \begin{align}
    \mathscr{F}_{\text{AD}}(\mathcal{A}) = \bm{P}(\bm{S}_n)\,\bm{P}(\bm{S}_{n-1})\cdots \bm{P}(\bm{S}_1)\,.
  \end{align}
\end{definition}
The product is row-stochastic because each factor is, so $\mathscr{F}_{\text{AD}}(\mathcal{A}) \in \mathcal{Y}$. For $n = 2$, $\mathscr{F}_{\text{AD}}$ recovers the alternating-diffusion operator of Section~\ref{subsec:ad}, applied to the Markov reformulations $\bm{P}(\bm{S}_1), \bm{P}(\bm{S}_2)$ of the two RSMs; for $n > 2$, it extends AD to an arbitrary number of views, in line with~\cite{lederman2018alternating}. The product is non-commutative, but the resulting operators share the same non-zero spectrum across orderings and induce equivalent diffusion geometries on the common variable~\cite{lederman2018alternating}, so the choice of ordering does not affect the analysis below.

\paragraph{Assumptions of alternating diffusion.}
Recall from Section~\ref{subsec:ad} that the AD theory assumes a hidden common variable $X$ and view-specific nuisance variables that are conditionally independent given $X$. In the neural-representation setting, we identify $X$ with the latent semantic content shared across the network and the view-specific factors with the layer-specific transformations of the forward pass. Therefore, strict conditional independence is an idealization in our setting, since consecutive layers are deterministic functions of one another and therefore strongly correlated. We will for some tasks select layers with some spacing along the network rather than adjacent ones, so that AD captures features of the representation as a whole rather than features specific to a single layer. For a remark about the assumption of local kernels in AD, see Appendix~\ref{app:local-nonlocal-kernel}.

\paragraph{AD-CKA and AD-DistCorr.}
Combining the AD fusion map with the Markov reformulations of CKA (Corollary~\ref{cor:cka}) and DistCorr (Corollary~\ref{cor:distcorr}) yields multi-representation analogues of both measures.

\begin{definition}[AD-CKA]
  \label{def:adcka}
  For two sets of sample-wise aligned representations $\mathcal{A}_1, \mathcal{A}_2 \in \mathfrak{A}$, the \emph{AD-CKA} similarity measure is defined as
  \begin{align}
    m_{\text{AD-CKA}}(\mathcal{A}_1, \mathcal{A}_2)
    = \frac{\hsic(\mathscr{F}_{\text{AD}}(\mathcal{A}_1),\, \mathscr{F}_{\text{AD}}(\mathcal{A}_2))}
           {\sqrt{\hsic(\mathscr{F}_{\text{AD}}(\mathcal{A}_1),\, \mathscr{F}_{\text{AD}}(\mathcal{A}_1))\,\hsic(\mathscr{F}_{\text{AD}}(\mathcal{A}_2),\, \mathscr{F}_{\text{AD}}(\mathcal{A}_2))}}\,,
  \end{align}
  where the outer $\bm{q}$-weighted centering of Theorem~\ref{thm:markovablemeasures} has been absorbed into the $\bm{H}$-centering already built into $\hsic$ (cf.~the proof of Corollary~\ref{cor:cka}).
\end{definition}

\begin{definition}[AD-DistCorr]
  \label{def:addistcorr}
  For two sets of sample-wise aligned representations $\mathcal{A}_1, \mathcal{A}_2 \in \mathfrak{A}$, the \emph{AD-DistCorr} similarity measure is defined as
  \begin{align}
    m_{\text{AD-DistCorr}}(\mathcal{A}_1, \mathcal{A}_2)
    = \frac{\langle \bm{H}\mathscr{F}_{\text{AD}}(\mathcal{A}_1)\bm{H},\, \bm{H}\mathscr{F}_{\text{AD}}(\mathcal{A}_2)\bm{H}\rangle_F}
           {\|\bm{H}\mathscr{F}_{\text{AD}}(\mathcal{A}_1)\bm{H}\|_F\,\|\bm{H}\mathscr{F}_{\text{AD}}(\mathcal{A}_2)\bm{H}\|_F}\,,
  \end{align}
  where the outer $\bm{q}$-weighted centering of Theorem~\ref{thm:markovablemeasures} has been written explicitly via the centering matrix $\bm{H}$ (cf.~the proof of Corollary~\ref{cor:distcorr}).
\end{definition}

For singleton sets $\mathcal{A}_1 = \{\bm{R}_1\}, \mathcal{A}_2 = \{\bm{R}_2\}$, the alternating-diffusion product reduces to a single factor $\bm{P}(\bm{S}_k)$ and AD-CKA and AD-DistCorr collapse to the standard CKA and DistCorr of Corollaries~\ref{cor:cka} and~\ref{cor:distcorr}.


\section{Experiments}
\label{sec:experiments}

\begin{table*}[t]
\centering
\scriptsize
\setlength{\tabcolsep}{3.5pt}
\renewcommand{\arraystretch}{1.05}
\caption{Spearman correlations on the ReSi prediction-grounded tests, comparing various baselines with the proposed diffusion-based measures. Best performance in bold.}
\label{tab:final-experiments-side-by-side}
\resizebox{\textwidth}{!}{%
\begin{tabular}{llrrrr@{\hspace{1.4em}}rrrr}
\toprule
 \multicolumn{2}{c}{Dataset} & \multicolumn{4}{c}{SST-2} & \multicolumn{4}{c}{ImageNet100} \\
 \multicolumn{2}{c}{Correlation to difference} & \multicolumn{2}{c}{Accuracy $\uparrow$} & \multicolumn{2}{c}{JSD (Output) $\uparrow$} & \multicolumn{4}{c}{Accuracy $\uparrow$} \\
 \multicolumn{2}{c}{Architecture} & BERT & SmolLM2 & BERT & SmolLM2 & RNet18 & RNet101 & VGG11 & VGG19 \\
\midrule
CCA & PWCCA~\cite{morcos2018insights} & -0.33 & -- & -0.32 & -- & -0.02 & -0.09 & -0.18 & 0.11 \\
 & SVCCA~\cite{raghu2017svcca} & -0.08 & 0.40 & 0.47 & 0.49 & 0.29 & -0.00 & -0.04 & -0.30 \\
Alignment & AlignCos~\cite{hamilton2016diachronic} & 0.02 & 0.44 & 0.49 & 0.39 & -0.08 & -0.01 & -0.13 & -0.12 \\
 & AngShape~\cite{williams2021shape} & -0.14 & 0.46 & 0.40 & 0.36 & 0.21 & 0.15 & -0.02 & 0.03 \\
 & HardCorr~\cite{li2015convergent} & -0.33 & 0.53 & -0.01 & 0.33 & 0.21 & -0.01 & -0.01 & -0.03 \\
 & LinReg~\cite{kornblith2019similarity} & -0.38 & 0.03 & 0.02 & 0.34 & 0.19 & 0.09 & -0.04 & 0.09 \\
 & OrthProc~\cite{ding2021grounding} & -0.14 & 0.46 & 0.40 & 0.36 & 0.21 & 0.15 & -0.02 & 0.03 \\
 & PermProc~\cite{williams2021shape} & -0.09 & 0.47 & 0.04 & 0.16 & 0.07 & 0.08 & 0.14 & -0.02 \\
 & ProcDist~\cite{williams2021shape} & 0.10 & 0.57 & 0.49 & 0.36 & 0.08 & 0.14 & 0.13 & 0.08 \\
 & SoftCorr~\cite{li2015convergent} & -0.33 & 0.55 & -0.01 & 0.37 & 0.27 & 0.04 & -0.03 & -0.10 \\
Neighbors & 2nd-Cos~\cite{hamilton2016cultural} & 0.30 & 0.10 & 0.49 & 0.18 & -0.08 & 0.05 & -0.20 & -0.18 \\
 & Jaccard~\cite{wang2022instability} & 0.17 & 0.31 & 0.54 & 0.24 & -0.11 & -0.04 & -0.22 & 0.07 \\
 & RankSim~\cite{wang2022instability} & 0.14 & 0.18 & \textbf{0.56} & 0.15 & 0.07 & 0.13 & -0.01 & 0.03 \\
Topology & IMD~\cite{tsitsulin2020shape} & -0.06 & 0.31 & 0.02 & 0.23 & 0.17 & 0.12 & 0.08 & -0.17 \\
 & RTD~\cite{barannikov2022rtd} & 0.05 & 0.32 & 0.33 & 0.27 & 0.09 & -0.20 & -0.19 & 0.05 \\
Statistic & ConcDiff~\cite{wang2022instability} & 0.12 & 0.06 & -0.02 & -0.11 & -0.11 & -0.04 & -0.11 & -0.13 \\
 & MagDiff~\cite{wang2022instability} & -0.13 & -0.10 & 0.11 & -0.07 & -0.16 & -0.06 & -0.07 & -0.12 \\
 & UnifDiff~\cite{wang2020alignment} & 0.35 & -0.13 & 0.38 & 0.05 & -0.18 & -- & 0.17 & -0.04 \\
RSM & CKA~\cite{kornblith2019similarity} & -0.06 & 0.48 & 0.48 & 0.52 & 0.36 & 0.16 & 0.03 & -0.20 \\
 & DistCorr~\cite{szekely2007distance} & -0.10 & 0.49 & 0.51 & 0.57 & 0.31 & 0.08 & 0.03 & -0.21 \\
 & EOS~\cite{may2019downstream} & -0.38 & -0.27 & -0.21 & 0.06 & 0.05 & 0.11 & -0.22 & 0.08 \\
 & GULP~\cite{boixadsera2022gulp} & -0.36 & -- & -0.30 & -- & 0.02 & 0.12 & -0.17 & 0.10 \\
 & RSA~\cite{kriegeskorte2008rsa} & -0.23 & 0.39 & 0.44 & 0.26 & 0.06 & 0.09 & 0.24 & -0.35 \\
 & RSMDiff~\cite{yin2018dimensionality} & 0.20 & 0.36 & 0.24 & -0.03 & 0.09 & 0.11 & -0.04 & -0.08 \\
\midrule
Diffusion-based & Multi-Scale CKA & -0.06 & 0.46 & 0.39 & 0.51 & \textbf{0.38} & \textbf{0.18} & 0.01 & -0.21 \\
(Ours) & AD-CKA & \textbf{0.43} & \textbf{0.58} & 0.40 & 0.66 & 0.17 & 0.06 & 0.27 & 0.10 \\
 & Multi-Scale DistCorr & -0.12 & 0.49 & 0.41 & 0.57 & 0.31 & 0.06 & 0.02 & -0.23 \\
 & AD-DistCorr & 0.25 & 0.55 & 0.38 & \textbf{0.74} & 0.17 & 0.06 & \textbf{0.32} & \textbf{0.19} \\
\bottomrule
\end{tabular}%
}
\end{table*}

The performance of similarity measures is very task-dependent. Therefore, we extensively test the performance of our similarity measures on two public benchmarks for representational similarity: ReSi \cite{klabunde2025resi} and GRS \cite{ding2021grounding}. In all experiments, we use the released code, trained models, datasets, and evaluation protocols of the original benchmarks, and compare networks under our multi-layer similarity formulation. This yields direct comparisons to existing measures, while evaluating our method in the network-to-network setting it is designed for.

\subsection{ReSi (Representational Similarity) benchmark}
The ReSi benchmark evaluates similarity measures on 14 architectures, which were trained with 10 seeds each on 7 datasets spanning vision, language, and graph tasks. The benchmark results are summarized in Table~\ref{tab:final-experiments-side-by-side}, where we compare against 24 baseline measures, with further results in Appendix~\ref{app:resi-results}. For each evaluation setting, the similarities for all 45 model pairs are evaluated with each measure, and then rank-correlated with the corresponding target metric.

\subsubsection{Experimental setting}
Our analysis centers on the tests grounded by prediction, which assess whether representational similarity tracks functional differences between models. Specifically, Test~1 in the benchmark measures correlation to the difference in accuracy, and Test~2 measures correlation to the difference in output. In both these tests, the models being compared differ only in their training seeds. Further experimental details are deferred to Appendix~\ref{app:resi-setup}. For language, we focus on transformers, considering BERT base~\cite{devlin2019bert} (an encoder model) and SmolLM2~\cite{allal2025smollm2} (a larger decoder-only model) on the SST-2~\cite{socher2013sst} dataset. For vision, we use ResNet18, ResNet101~\cite{he2016resnet}, VGG11, and VGG19~\cite{simonyan2015vgg} on the ImageNet100~\cite{deng2009imagenet} dataset.

\subsubsection{Results and discussion}

\paragraph{Language transformers.}
In this setting, the preferred layer selection depends on the test. In Test~1, selecting layers broadly distributed across the network is most effective. Under this configuration, AD-CKA achieves the highest correlation among all measures in the ReSi benchmark for both BERT and SmolLM2, indicating that aggregating information across depth yields a substantially stronger network-level similarity signal for overall accuracy than standard single-layer measures.

In Test~2, the pattern changes: the best performance is obtained by concentrating on the deepest hidden layers rather than sampling broadly across the network. In this regime, AD-DistCorr performs particularly well and attains the best score among all ReSi measures for SmolLM2.

Taken together, these results show that the optimal layer selection depends on the functional criterion being grounded: broad depth coverage is most useful for accuracy, whereas the deepest layers are most informative for agreement at the level of sample-wise outputs. The consistency of these patterns across both language architectures suggests that they reflect a genuine property of language transformers, further supporting the value of multi-layer aggregation in this architecture class.

\paragraph{Vision models.}
In vision on Test~1, we observe a more architecture-dependent pattern. For ResNet18 and ResNet101, Multi-Scale CKA achieves the best result among all ReSi measures for both models. Unlike in language, the more involved multi-layer variants do not provide additional benefit here, suggesting that for residual architectures, an appropriately processed final representation may already capture most of the relevant similarity signal.

For the VGG architectures, the picture differs. On both VGG11 and VGG19, AD-DistCorr attains the best score among all baseline measures. Thus, while the strongest variant depends on the architecture family, the overall conclusion remains consistent: the proposed measures are highly competitive across vision models and outperform all benchmark baselines in several important settings.

\paragraph{Conclusions.}
1)~The proposed measures achieve state-of-the-art benchmark performance in several key settings across both language and vision. This indicates that multi-layer aggregation is especially valuable for tracking functional differences between models, and that the proposed framework captures network-level structure that is not accessible through standard single-layer comparisons. 2)~The most effective form of aggregation depends on the architectural regime: multi-layer aggregation is especially beneficial for language transformers, and also some vision models, whereas our simpler variants are often sufficient, and sometimes preferable, for residual vision architectures.

\subsection{GRS (Grounding Representational Similarity with statistical testing) benchmark}

GRS emphasizes out-of-distribution (OOD) behavior, and we focus on its Benchmark~4, which is naturally aligned with our goal of network-to-network comparison. For experimental details, see Appendix \ref{app:grs}.
Benchmark~4 is the hardest setting in GRS, varying both pretraining and fine-tuning seeds and targeting OOD accuracy on the Antonymy and Numerical stress tests~\cite{naik2018stress}.
\begin{wraptable}{r}{0.35\textwidth}
\vspace{-1em}
\centering
\scriptsize
\caption{Performance on GRS.}
\label{tab:language-small}
\begin{tabular}{lcc@{\hspace{0.8em}}cc}
\toprule
& \multicolumn{2}{c}{\shortstack{Antonym\\stress test}}
& \multicolumn{2}{c}{\shortstack{Numerical\\stress test}} \\
\cmidrule(lr){2-3}\cmidrule(lr){4-5}
Method & $\tau$ & $\rho$ & $\tau$ & $\rho$ \\
\midrule
Procrustes & 0.24 & 0.18 & 0.07 & 0.05 \\
CKA        & 0.23 & 0.16 & 0.12 & 0.08 \\
PWCCA      & 0.20 & 0.15 & 0.03 & 0.02 \\
AD-CKA     & \textbf{0.41} & \textbf{0.29} & \textbf{0.20} & \textbf{0.14} \\
\bottomrule
\end{tabular}
\vspace{-0.5em}
\end{wraptable}
This setting complements the prediction-grounded tests in ReSi. There, the main signal is in-distribution performance; here, it is OOD accuracy, which is particularly sensitive to initializations~\cite{mccoy-etal-2020-berts}.
Since this benchmark also uses a BERT model (medium), it further strengthens our broader language transformer narrative.

AD-CKA achieves the best results on both tasks by a comfortable margin, see Table~\ref{tab:language-small}. This is especially notable because \cite{ding2021grounding} explicitly presents Benchmark~4 as a challenge set on which no measure attains strong correlation, positioning it as a benchmark intended to spur further progress.

\section{Limitations}
\label{sec:lim}

ReSi found that no similarity measure performs well across the board, and ours are no exception, e.g., on graphs, our results are mixed (Appendix~\ref{app:resi-results}). Our multi-layer framework is limited in the number of layers that can be aggregated, as products of many Markov matrices tend toward the rank-1 stationary matrix (Appendix~\ref{app:resi-stability}). Extensions of AD, e.g.~\cite{katz2019alternating}, may help mitigate this.

\section*{Acknowledgments}

This work is supported by the Wallenberg AI, Autonomous Systems and Software Program (WASP) funded by the Knut and Alice Wallenberg Foundation. The computations were enabled by resources provided by the National Academic Infrastructure for Supercomputing in Sweden (NAISS), partially funded by the Swedish Research Council through grant agreement no. 2022-06725.

\newpage
\renewcommand*{\bibfont}{\normalfont\footnotesize}
\printbibliography


\newpage
\appendix

\section{Proofs}
\label{sec:theorem1proof}

In this appendix, we provide the proofs that were omitted in the main part of the paper. For convenience, we restate the statements to be proven here.

\markovablemeasures*

\begin{proof}

For $k\in\{1,2\}$, write $\bm{A}_k := \bm{C_q}(\bm{S}_k)$. By~(\ref{eq:4}), $\bm{A}_k\mathbf{1}=\bm{0}$.

\paragraph{$\bm{P}(\bm{S}_k)$ is nonnegative.}
The $(i,j)$-th entry of $\bm{P}(\bm{S}_k)$ is
\begin{align}
    (\bm{P}(\bm{S}_k))_{ij} = q_j + \alpha(\bm{S}_k)\,(\bm{A}_k)_{ij}\,.
\end{align}
Assume $\bm{A}_k\neq\bm{0}$ (otherwise $\bm{P}(\bm{S}_k)=\mathbf{1}\bm{q}^\top$, which is trivially row-stochastic, and the remainder of the proof follows with $\alpha(\bm{S}_k)$ chosen as any positive number). Since each $q_j>0$ and the minimum in the definition of $\alpha(\bm{S}_k)$ is taken over a nonempty set of strictly positive numbers, $\alpha(\bm{S}_k)>0$.

If $(\bm{A}_k)_{ij}\ge 0$, then $(\bm{P}(\bm{S}_k))_{ij}\ge q_j>0$. If $(\bm{A}_k)_{ij}<0$, then by definition of $\alpha(\bm{S}_k)$,
\begin{align}
    \alpha(\bm{S}_k)
    \le \frac{q_j}{|(\bm{A}_k)_{ij}|}
    = \frac{q_j}{-(\bm{A}_k)_{ij}}\,,
\end{align}
which rearranges to $q_j + \alpha(\bm{S}_k)(\bm{A}_k)_{ij}\ge 0$. Hence $\bm{P}(\bm{S}_k)\ge\bm{0}$ entrywise.

\paragraph{$\bm{P}(\bm{S}_k)$ is row-stochastic.}
Since $\bm{q}^\top\mathbf{1}=1$ and $\bm{A}_k\mathbf{1}=\bm{0}$ by~(\ref{eq:4}),
\begin{align}
    \bm{P}(\bm{S}_k)\mathbf{1}
    = \mathbf{1}(\bm{q}^\top\mathbf{1}) + \alpha(\bm{S}_k)\bm{A}_k\mathbf{1}
    = \mathbf{1}\,.
\end{align}
Combined with nonnegativity, this shows that $\bm{P}(\bm{S}_k)$ is a Markov matrix.

\paragraph{Markov reformulation of $m_{\text{RSM}}$.}
By linearity of $\bm{C_q}$,
\begin{align}
    \bm{C_q}(\bm{P}(\bm{S}_k)) = \bm{C_q}(\mathbf{1}\bm{q}^\top) + \alpha(\bm{S}_k)\,\bm{C_q}(\bm{A}_k)\,.
\end{align}
By~(\ref{eq:4}), $\bm{C_q}(\mathbf{1}\bm{q}^\top)=\bm{0}$ and
\begin{align}
    \bm{C_q}(\bm{A}_k) = \bm{C_q}(\bm{C_q}(\bm{S}_k)) = \bm{C_q}(\bm{S}_k) = \bm{A}_k\,.
\end{align}
Therefore,
\begin{align}
    \bm{C_q}(\bm{P}(\bm{S}_k)) = \alpha(\bm{S}_k)\,\bm{C_q}(\bm{S}_k)\,.
\end{align}
Since $\alpha(\bm{S}_1),\alpha(\bm{S}_2)>0$ and $\psi$ is invariant to separate positive rescalings,
\begin{align}
    \psi(\bm{C_q}(\bm{P}(\bm{S}_1)),\bm{C_q}(\bm{P}(\bm{S}_2)))
    &= \psi\bigl(\alpha(\bm{S}_1)\bm{C_q}(\bm{S}_1),\,\alpha(\bm{S}_2)\bm{C_q}(\bm{S}_2)\bigr) \\
    &= \psi(\bm{C_q}(\bm{S}_1),\bm{C_q}(\bm{S}_2))
    = m_{\text{RSM}}(\bm{S}_1,\bm{S}_2)\,.
\end{align}
\end{proof}

\centeringoplemma*
\begin{proof}
Set $\bm{E_{q}} := \bm{I} - \mathbf{1}\bm{q}^\top$, so that $\bm{C_q}(\bm{M}) = \bm{E_{q}}\bm{M}\bm{E_{q}}$. Since $\bm{q}$ is a probability vector, $\bm{q}^\top\mathbf{1} = 1$, hence
\begin{align}
    \bm{E_{q}}\mathbf{1} &= \mathbf{1} - \mathbf{1}\bm{q}^\top\mathbf{1} = \mathbf{1} - \mathbf{1} = \bm{0}\,, \\
    \bm{q}^\top\bm{E_{q}} &= \bm{q}^\top - \bm{q}^\top\mathbf{1}\bm{q}^\top = \bm{q}^\top - \bm{q}^\top = \bm{0}^\top\,.
\end{align}

\emph{Idempotence, $\bm{C_q}^2 = \bm{C_q}$.} Using $\bm{q}^\top\bm{E_{q}} = \bm{0}^\top$,
\begin{align}
    \bm{E_{q}}^2 = (\bm{I} - \mathbf{1}\bm{q}^\top)\bm{E_{q}} = \bm{E_{q}} - \mathbf{1}\bm{q}^\top\bm{E_{q}} = \bm{E_{q}}\,,
\end{align}
hence for any $\bm{M}\in\mathbb{R}^{N\times N}$,
\begin{align}
    \bm{C_q}(\bm{C_q}(\bm{M})) = \bm{E_{q}}(\bm{E_{q}}\bm{M}\bm{E_{q}})\bm{E_{q}} = \bm{E_{q}}^2\bm{M}\bm{E_{q}}^2 = \bm{E_{q}}\bm{M}\bm{E_{q}} = \bm{C_q}(\bm{M})\,.
\end{align}

\emph{Row-annihilation, $\bm{C_q}(\bm{M})\mathbf{1} = \bm{0}$ for any $\bm{M}\in\mathbb{R}^{N\times N}$.} Using $\bm{E_{q}}\mathbf{1}=\bm{0}$,
\begin{align}
    \bm{C_q}(\bm{M})\mathbf{1} = \bm{E_{q}}\bm{M}\bm{E_{q}}\mathbf{1} = \bm{E_{q}}\bm{M}\bm{0} = \bm{0}\,.
\end{align}

\emph{Annihilation of $\mathbf{1}\bm{q}^\top$, $\bm{C_q}(\mathbf{1}\bm{q}^\top) = \bm{0}$.} Using $\bm{E_{q}}\mathbf{1} = \bm{0}$,
\begin{align}
    \bm{C_q}(\mathbf{1}\bm{q}^\top) = \bm{E_{q}}\,\mathbf{1}\bm{q}^\top\bm{E_{q}} = \bm{0}\,\bm{q}^\top\bm{E_{q}} = \bm{0}\,.
\end{align}

Finally, for $\bm{q} = \frac{1}{N}\bm{1}$, we have $\bm{E_{q}} = \bm{I} - \frac{1}{N}\mathbf{1}\mathbf{1}^\top = \bm{H}$, so $\bm{C_q}(\bm{M}) = \bm{H}\bm{M}\bm{H}$.
\end{proof}

\ckacorollary*
\begin{proof}
  Let $\bm{C}(\bm{M}) = \bm{H}\bm{M}\bm{H}$. Then, the HSIC estimator~(\ref{eq:2}) satisfies
  \begin{align}
    \hsic(\bm{C}(\bm{S}_1), \bm{C}(\bm{S}_{2}))
    =\frac{\tr(\bm{H}\bm{S}_1\bm{H}\bm{H}\bm{H}\bm{S}_{2}\bm{H}\bm{H})}{(N-1)^2}
    =\frac{\tr(\bm{S}_1\bm{H}\bm{S}_{2}\bm{H})}{(N-1)^2}
    =\hsic(\bm{S}_1, \bm{S}_2)\,,
  \end{align}
  since $\bm{H}^{\top}=\bm{H}$ and $\bm{H}^{2}=\bm{H}$. With
  \begin{align}
    \psi(\bm{S}_1, \bm{S}_2) = \frac{\hsic(\bm{S}_1, \bm{S}_2)}{\sqrt{\hsic(\bm{S}_1, \bm{S}_1)\hsic(\bm{S}_2, \bm{S}_2)}}\,,
  \end{align}
  the CKA measure can be written as
  \begin{align}
    m_{\text{CKA}}(\bm{S}_{1},\bm{S}_{2})=\psi(\bm{S}_{1},\bm{S}_{2})=\psi(\bm{C}(\bm{S}_{1}),\bm{C}(\bm{S}_{2}))\,.
  \end{align}
  By Lemma~\ref{lem:centeringop}, $\bm{C}$ satisfies~(\ref{eq:4}). Since $\hsic(a\bm{A},b\bm{B})=ab\hsic(\bm{A},\bm{B})$, $\psi$ is invariant to rescalings. Therefore, the assumptions of Theorem~\ref{thm:markovablemeasures} are all satisfied.
\end{proof}

\distcorrcorollary*
\begin{proof}
  Let $\bm{C}(\bm{M}) = \bm{H}\bm{M}\bm{H}$. In terms of the Frobenius inner product $\langle \bm{A},\bm{B} \rangle_{F}=\tr(\bm{A}^{\top}\bm{B})$, the distance covariance~(\ref{eq:6}) for mean-centered RSMs can be written as
  \begin{align}
    \dcov^{2}(\bm{S}_1, \bm{S}_2)
    =\frac{1}{N^2}\langle\bm{S}_1, \bm{S}_2\rangle_F\,.
  \end{align}
  Therefore, for uncentered RSMs, we have
  \begin{align}
    \dcov^{2}(\bm{S}_1, \bm{S}_2)=\frac{1}{N^2}\langle\bm{C}(\bm{S}_1), \bm{C}(\bm{S}_2)\rangle_F\,.
  \end{align}
  With
  \begin{align}
    \psi(\bm{A}, \bm{B}) = \frac{\langle \bm{A}, \bm{B}\rangle_F}{\|\bm{A}\|_F\,\|\bm{B}\|_F}\,,
  \end{align}
  the distance correlation measure can be written as
  \begin{align}
    m_{\text{DistCorr}}(\bm{S}_{1},\bm{S}_{2})=\psi(\bm{C}(\bm{S}_{1}),\bm{C}(\bm{S}_{2}))\,.
  \end{align}
  By Lemma~\ref{lem:centeringop}, $\bm{C}$ satisfies~(\ref{eq:4}). Since $\psi$ is scale-invariant, the assumptions of Theorem~\ref{thm:markovablemeasures} are all satisfied.
\end{proof}


\section{Local vs.\ nonlocal Markov kernels}
\label{app:local-nonlocal-kernel}
The analysis of AD in~\cite{lederman2018alternating} is stated for \emph{local} Markov kernels, i.e.\ row-stochastic matrices derived from a Gaussian affinity with small bandwidth, so that a single diffusion step stays within a small neighborhood in observation space. Our Markov matrices are nonlocal in this sense: by Theorem~\ref{thm:markovablemeasures}, $\alpha(\bm{S})$ bounds the magnitude of $\alpha(\bm{S})\bm{C_{q}}(\bm{S})$ entrywise by $\bm{q}$, so $\bm{P}_{ij}\in[0,\,2q_j]$ fluctuates around the uniform background $q_j$ and a single step of the induced random walk is close to uniform rather than localized to a neighborhood, even when $\bm{S}$ itself decays with distance, as for RBF CKA. The proofs of the AD theorems, however, do not rely on locality, and the authors note that the assumption is made for simplicity and can be relaxed~\cite[Sec.~5]{lederman2018alternating}. Our empirical results (Section~\ref{sec:experiments}) can therefore be read as a validation that AD-style fusion remains useful for nonlocal Markov kernels.
\section{Further details and results on ReSi benchmark}
\label{app:resi}

\subsection{Evaluation setup}
\label{app:resi-setup}
ReSi evaluates representational similarity primarily on the final hidden layer representations, motivated by the fact that these representations feed directly into the classifier and therefore keep the prediction-grounded tests tied to the model’s observed behavior.

For Tests 1 and 2, ReSi grounds representational similarity in predictive behavior by considering a set $\mathcal{F}$ of models trained under the same architecture-dataset setting, differing only in training seed. Let $\mathbf{R}_f$ denote the representation of model $f \in \mathcal{F}$, let $m(\mathbf{R}_f,\mathbf{R}_{f'})$ be a similarity measure, and let $\mathbf{O}_f$ be the output probabilities of $f$ on a fixed evaluation set. ReSi evaluates whether representation similarity tracks differences in model predictions.

\paragraph{Test 1: Correlation to Accuracy Difference.}
For each model pair $(f,f')$, ReSi computes the absolute difference in
classification accuracy,
\begin{align}
\Delta_{\mathrm{Acc}}(f,f') =
\left|\mathrm{Acc}(\mathbf{O}_f,\mathbf{y}) -
\mathrm{Acc}(\mathbf{O}_{f'},\mathbf{y})\right|.
\end{align}
It then reports the Spearman rank correlation between the set of representational similarities $m(\mathbf{R}_f,\mathbf{R}_{f'})$ and the set of accuracy
differences $\Delta_{\mathrm{Acc}}(f,f')$ over all model pairs.

\paragraph{Test 2: Correlation to Output Difference.}
ReSi refines this idea by comparing instance-wise predictions. For each pair $(f,f')$, it computes both the disagreement of the hard predictions,
\begin{align}
\Delta_{\mathrm{Dis}}(f,f') &=
\frac{1}{n}\sum_{i=1}^{n}\mathbf{1}\!\left[
\arg\max \mathbf{O}_f^{(i)}\neq
\arg\max \mathbf{O}_{f'}^{(i)}\right],
\end{align}
and the average Jensen-Shannon divergence of the class-probability vectors,
\begin{align}
\Delta_{\mathrm{JSD}}(f,f') &=
\frac{1}{n}\sum_{i=1}^{n}
\mathrm{JSD}\!\left(\mathbf{O}_f^{(i)} \,\|\, \mathbf{O}_{f'}^{(i)}\right).
\end{align}
As in Test~1, the benchmark reports the Spearman rank correlation between
$m(\mathbf{R}_f,\mathbf{R}_{f'})$ and each of these output-difference quantities
over all model pairs.

\paragraph{Language model details.} ReSi uses the CLS token representations for BERT and ALBERT, and the final prompt token for SmolLM2.

\subsubsection{Integration of proposed similarity measures}
For our multi-scale variants, integration is direct, since they operate on the same single-layer inputs as the original benchmark. For our multi-layer variants, however, the natural object to compare is an entire network rather than a single layer. We therefore retain ReSi’s models, datasets, and evaluation protocols, but replace the single-layer representation with a fixed subset of hidden layer representations extracted from both networks. Using Definition~\ref{def:adfusion}, we aggregate these layers into one Markov matrix per network and then compute the proposed multi-layer similarities using Definitions~\ref{def:adcka} and~\ref{def:addistcorr}. The task definitions and scoring procedures remain otherwise unchanged. We therefore regard the resulting comparison to ReSi’s original single-layer measures as fair and informative. It directly tests whether aggregating information across multiple depths yields a stronger network-level similarity notion under the same benchmark conditions, and whether our multi-layer framework provides a meaningful aggregated representation of the selected set of single-layer representations. Empirically, this often leads to substantial gains on Tests~1 and~2. At the same time, our multi-step variants, which are considerably cheaper and remain fully aligned with ReSi’s original single-layer setup, already outperform several baseline measures and, in some cases, attain the best score among all measures considered.

For the multi-scale variants, $(m^{2}_{MS-DistCorr}(\bm{S_1}, \bm{S}_2) + 2 m_{DistCorr}(\bm{S_1}, \bm{S}_2))/3$ is used as the final similarity score.

\subsubsection{Layers selected for each test}
See Table~\ref{tab:app-layer-set}. For the graphs domain, all representations are selected.
\begin{table*}[t]
\centering
\scriptsize
\setlength{\tabcolsep}{3.5pt}
\renewcommand{\arraystretch}{1.05}
\caption{Representations that are elements of the set $\mathcal{A}$ for the setting being considered. We select from the set of layers extracted by the ReSi benchmark (e.g., 13 representations are extracted from BERT base). For language models, the 0th extracted representation is the embedding + positional encoding layer.}
\label{tab:app-layer-set}
\resizebox{0.78\textwidth}{!}{%
\begin{tabular}{llrrrr@{\hspace{1.4em}}rrrr}
\toprule
 Domain & Dataset & Correlation & Architecture & Representation indices \\
 & & to difference & & \\
 \midrule
 Language & SST-2 and MNLI & Accuracy & BERT base & [1, 3, 6, 9, 12] \\
 & & & ALBERT & [1, 3, 6, 9, 12] \\
 & & & SmolLM2 & [17, 21, 24] \\
 & & JSD (Output) & BERT base & [11, 12] \\
 & & & ALBERT & [11, 12] \\
 & & & SmolLM2 & [21, 22, 23, 24] \\
 Vision & ImageNet100 & Accuracy & ResNet18 & [1, 4, 7, 9] \\
 & & & ResNet34 & [11, 14, 17] \\
 & & & ResNet101 & [0, 5, 11, 17, 23, 28, 34] \\
 & & & VGG11 & [4, 6, 8] \\
 & & & VGG19 & [0, 3, 6, 9, 13, 16] \\
 & & & ViT B32 & [7, 10, 12] \\
 & & & ViT L32 & [7, 10, 12] \\
 & & JSD (Output) & ResNet18 & [8, 9] \\
 & & & ResNet34 & [15, 16, 17] \\
 & & & ResNet101 & [30, 31, 32, 33, 34] \\
 & & & VGG11 & [7, 8] \\
 & & & VGG19 & [14, 15, 16] \\
 & & & ViT B32 & [11, 12] \\
 & & & ViT L32 & [11, 12] \\
\bottomrule
\end{tabular}%
}
\end{table*}

\subsection{Stability analysis}
\label{app:resi-stability}
A practical limitation of our multi-layer construction is that forming long products of Markov matrices, $\bm{P}_1\bm{P}_2\bm{P}_3...\bm{P}_n$, can lead to numerical degeneration. As the number of factors increases, the product tends toward the uniform matrix, corresponding to the stationary distribution, and the informative signal becomes dominated by floating-point noise. To mitigate this effect, we perform all computations in double-precision floats and restrict the size of the set $\mathcal{A}$ to at most 8 representations.

With these implementation choices, the method is numerically stable and yields meaningful scores. In particular, the final similarity values are unchanged under random perturbations of the input representations of magnitude up to $10^{-7}$.

A possible way to address this limitation more fundamentally is to use extensions of alternating diffusion designed for genuine multi-view settings with more than two views, for instance, \cite{katz2019alternating}. Such methods may provide a more natural and scalable alternative to repeated pairwise composition, and, since our results already demonstrate the effectiveness of the core multi-view learning principle for aggregating multiple representations, we view this as a particularly promising direction for future work.

\subsection{Compute resources}
\label{app:resi-compute}
All experiments were conducted on a compute cluster, primarily using a single NVIDIA A40 or a single NVIDIA V100 GPU at a time. The overall computational requirements are modest: the experiments are relatively inexpensive to run and can also be executed on a single T4 GPU.

\subsection{Additional results}
\label{app:resi-results}
CKA denotes CKA with a linear kernel. RBF-CKA denotes CKA with the RBF kernel. Refer to Section~\ref{subsec:rep_sim_m}.

See Tables \ref{tab:appendix-table19-extended}, \ref{tab:appendix-table21-extended}, \ref{tab:appendix-table13-extended}, \ref{tab:appendix-table7-extended}, and \ref{tab:appendix-table3-extended}.

Asterisks on ReSi entries denote statistical significance for the prediction-grounded correlation results: $*$ indicates significance at the 5\% level and $**$ indicates significance at the 1\% level.

\begin{table*}[t]
\centering
\scriptsize
\setlength{\tabcolsep}{3.5pt}
\renewcommand{\arraystretch}{1.05}
\caption{Results of Test 1 (Correlation to Accuracy Difference) for the vision domain on ImageNet-100. Best performance in bold.}
\label{tab:appendix-table19-extended}
\resizebox{\textwidth}{!}{%
\begin{tabular}{llrrrrrrr}
\toprule
\multicolumn{2}{c}{Dataset} & \multicolumn{7}{c}{ImageNet100} \\
\multicolumn{2}{c}{Arch.} & ResNet18 & ResNet34 & ResNet101 & VGG11 & VGG19 & ViT B32 & ViT L32 \\
\midrule
CCA & PWCCA & -0.02 & -0.20 & -0.09 & -0.18 & 0.11 & -0.08 & 0.09 \\
    & SVCCA & 0.29 & 0.27 & -0.00 & -0.04 & -0.30 & -0.01 & -0.17 \\
Alignment & AlignCos & -0.08 & -0.35 & -0.01 & -0.13 & -0.12 & 0.07 & 0.05 \\
    & AngShape & 0.21 & -0.16 & 0.15 & -0.02 & 0.03 & 0.07 & 0.06 \\
    & HardCorr & 0.21 & 0.13 & -0.01 & -0.01 & -0.03 & 0.35 & -0.17 \\
    & LinReg & 0.19 & -0.11 & 0.09 & -0.04 & 0.09 & 0.15 & 0.05 \\
    & OrthProc & 0.21 & -0.16 & 0.15 & -0.02 & 0.03 & 0.07 & 0.06 \\
    & PermProc & 0.07 & 0.09 & 0.08 & 0.14 & -0.02 & -0.06 & -0.33 \\
    & ProcDist & 0.08 & 0.00 & 0.14 & 0.13 & 0.08 & 0.16 & 0.05 \\
    & SoftCorr & 0.27 & 0.08 & 0.04 & -0.03 & -0.10 & 0.36 & -0.19 \\
RSM & CKA & 0.36 & -0.07 & 0.16 & 0.03 & -0.20 & -0.26 & 0.05 \\
    & RBF-CKA & 0.25 & -0.03 & -0.00 & -0.03 & -0.06 & -0.30 & 0.10 \\
    & DistCorr & 0.31 & -0.08 & 0.08 & 0.03 & -0.21 & -0.26 & 0.03 \\
    & EOS & 0.05 & -0.17 & 0.11 & -0.22 & 0.08 & \textbf{0.47} & 0.03 \\
    & GULP & 0.02 & -0.18 & 0.12 & -0.17 & 0.10 & 0.18 & 0.04 \\
    & RSA & 0.06 & -0.17 & 0.09 & 0.24 & -0.35 & -0.12 & -0.11 \\
    & RSMDiff & 0.09 & -0.10 & 0.11 & -0.04 & -0.08 & 0.01 & -0.06 \\
Neighbors & 2nd-Cos & -0.08 & -0.15 & 0.05 & -0.20 & -0.18 & -0.22 & 0.17 \\
    & Jaccard & -0.11 & -0.13 & -0.04 & -0.22 & 0.07 & -0.01 & 0.25 \\
    & RankSim & 0.07 & 0.04 & 0.13 & -0.01 & 0.03 & 0.17 & \textbf{0.35} \\
Topology & IMD & 0.17 & -0.20 & 0.12 & 0.08 & -0.17 & -0.23 & -0.19 \\
    & RTD & 0.09 & 0.12 & -0.20 & -0.19 & 0.05 & -0.11 & 0.10 \\
Statistic & ConcDiff & -0.11 & \textbf{0.34} & -0.04 & -0.11 & -0.13 & 0.00 & 0.18 \\
    & MagDiff & -0.16 & 0.02 & -0.06 & -0.07 & -0.12 & 0.07 & 0.15 \\
    & UnifDiff & -0.18 & -- & -- & 0.17 & -0.04 & -- & -- \\
\midrule
Diffusion-based & Multi-Scale CKA & \textbf{0.38} & 0.09 & \textbf{0.18} & 0.01 & -0.21 & -0.23 & -0.01 \\
(Ours) & AD-CKA & 0.17 & -0.05 & 0.06 & 0.27 & 0.10 & -0.26 & -0.22 \\
    & Multi-Scale RBF-CKA & 0.25 & -0.01 & -0.06 & -0.03 & -0.08 & -0.28 & 0.03 \\
    & AD-RBF-CKA & 0.14 & -0.03 & 0.11 & 0.08 & \textbf{0.25} & -0.47 & -0.08 \\
    & Multi-Scale DistCorr & 0.31 & 0.08 & 0.06 & 0.02 & -0.23 & -0.23 & -0.03 \\
    & AD-DistCorr & 0.17 & -0.06 & 0.06 & \textbf{0.32} & 0.19 & -0.28 & -0.19 \\
\bottomrule
\end{tabular}%
}
\end{table*}

\begin{table*}[t]
\centering
\scriptsize
\setlength{\tabcolsep}{3.5pt}
\renewcommand{\arraystretch}{1.05}
\caption{Results of Test 2 (Correlation to Output Difference) for the vision domain on ImageNet-100.}
\label{tab:appendix-table21-extended}
\resizebox{\textwidth}{!}{%
\begin{tabular}{llrrrrrrr@{\hspace{1.2em}}rrrrrrr}
\toprule
\multicolumn{2}{c}{Test} & \multicolumn{7}{c}{JSD Correlation$\uparrow$} & \multicolumn{7}{c}{Disagreement Correlation$\uparrow$} \\
\multicolumn{2}{c}{Dataset} & \multicolumn{7}{c}{ImageNet100} & \multicolumn{7}{c}{ImageNet100} \\
\multicolumn{2}{c}{Arch.} & RNet18 & RNet34 & RNet101 & VGG11 & VGG19 & ViT B32 & ViT L32 & RNet18 & RNet34 & RNet101 & VGG11 & VGG19 & ViT B32 & ViT L32 \\
\midrule
CCA & PWCCA & 0.13 & 0.15 & 0.15 & -0.13 & 0.15 & 0.21 & -0.24* & 0.27** & 0.33** & 0.06 & -0.12 & -0.36** & 0.02 & -0.18 \\
 & SVCCA & 0.21 & -0.00 & 0.25 & -0.11 & 0.16 & 0.05 & \textbf{0.18} & 0.39** & 0.07 & 0.14 & 0.03 & 0.01 & -0.06 & 0.07 \\
Alignment & AlignCos & 0.08 & 0.05 & 0.38* & 0.10 & -0.20 & -0.22 & -0.06 & 0.19 & \textbf{0.50**} & 0.17 & 0.16 & -0.13 & -0.08 & 0.00 \\
 & AngShape & 0.24 & 0.22 & 0.34* & -0.01 & 0.19 & -0.26 & -0.15 & 0.24 & 0.40** & 0.19 & -0.14 & -0.33* & -0.11 & -0.06 \\
 & HardCorr & 0.28 & 0.31* & 0.02 & -0.22 & -0.05 & 0.03 & -0.24 & 0.28 & 0.06 & -0.04 & -0.15 & -0.18 & 0.27 & -0.16 \\
 & LinReg & 0.21* & 0.21* & \textbf{0.41**} & -0.01 & 0.25* & -0.13 & -0.14 & 0.19 & 0.25 & 0.29** & -0.17 & -0.24* & 0.04 & -0.07 \\
 & OrthProc & 0.24 & 0.22 & 0.34* & -0.02 & 0.19 & -0.26 & -0.15 & 0.24 & 0.40** & 0.19 & -0.14 & -0.33* & -0.11 & -0.06 \\
 & PermProc & 0.18 & 0.18 & 0.27 & -0.18 & 0.06 & 0.36* & 0.06 & 0.13 & 0.25 & -0.04 & 0.02 & 0.20 & \textbf{0.37*} & 0.10 \\
 & ProcDist & 0.10 & 0.14 & 0.39* & -0.05 & 0.27 & -0.05 & 0.02 & 0.08 & -0.08 & 0.11 & -0.10 & -0.07 & -0.07 & 0.08 \\
 & SoftCorr & \textbf{0.45**} & 0.27 & 0.11 & -0.04 & -0.16 & 0.01 & -0.31* & 0.47** & -0.13 & -0.07 & -0.03 & -0.29 & 0.27 & -0.16 \\
RSM & CKA & 0.30* & 0.08 & 0.30* & -0.13 & -0.06 & 0.04 & -0.07 & 0.37* & 0.08 & 0.22 & 0.01 & -0.24 & 0.00 & -0.02 \\
 & RBF-CKA & 0.17 & -0.11 & 0.28 & -0.07 & 0.20 & 0.01 & -0.24
& 0.31 & -0.13 & 0.35 & -0.01 & -0.21 & 0.02 & -0.19 \\
 & DistCorr & 0.26 & 0.05 & 0.31* & -0.10 & 0.04 & 0.05 & -0.12 & 0.36* & 0.01 & 0.28 & 0.00 & -0.25 & 0.02 & -0.05 \\
 & EOS & 0.09 & \textbf{0.49**} & 0.33* & -0.11 & 0.15 & -0.18 & -0.28 & 0.11 & 0.25 & 0.14 & -0.31* & -0.41** & 0.01 & -0.15 \\
 & GULP & 0.07 & \textbf{0.49**} & 0.35* & -0.05 & 0.15 & -0.05 & -0.28 & 0.10 & 0.26 & 0.13 & -0.27 & -0.41** & 0.19 & \textbf{0.15} \\
 & RSA & 0.12 & 0.18 & 0.09 & -0.18 & 0.19 & -0.11 & -0.20 & 0.19 & 0.33* & 0.10 & -0.05 & 0.04 & 0.09 & -0.04 \\
 & RSMDiff & -0.41** & -0.22 & 0.30* & -0.27 & 0.07 & 0.02 & -0.28 & -0.17 & -0.20 & 0.18 & -0.01 & -0.03 & -0.21 & -0.17 \\
Neighbors & 2nd-Cos & -0.13 & 0.16 & 0.29 & 0.07 & -0.29 & 0.43* & -0.35* & -0.21 & 0.45* & 0.10 & 0.11 & -0.07 & 0.19 & -0.27 \\
 & Jaccard & 0.36* & 0.26 & 0.32* & 0.05 & \textbf{0.33*} & \textbf{0.47**} & -0.30* & 0.25 & 0.47** & 0.23 & 0.14 & -0.11 & 0.34* & -0.18 \\
 & RankSim & -0.15 & -0.00 & 0.22 & 0.01 & 0.05 & 0.25 & -0.09 & -0.10 & -0.05 & 0.02 & 0.01 & -0.32* & 0.14 & -0.33* \\
Topology & IMD & -0.11 & 0.08 & 0.21 & \textbf{0.20} & 0.11 & 0.41** & 0.07 & -0.07 & 0.00 & 0.26 & 0.26 & 0.02 & 0.23 & 0.07 \\
 & RTD & -0.18 & 0.02 & 0.17 & -0.02 & -0.21 & 0.20 & -0.27 & -0.08 & -0.06 & 0.33* & \textbf{0.27} & -0.29 & 0.36* & -0.23 \\
Statistic & ConcDiff & -0.29* & 0.24 & -0.11 & -0.17 & -0.13 & -0.11 & -0.37* & -0.09 & 0.00 & -0.21 & -0.11 & -0.06 & -0.08 & -0.29 \\
 & MagDiff & -0.38* & -0.20 & 0.02 & -0.16 & -0.28 & 0.02 & -0.32* & -0.17 & -0.22 & -0.05 & -0.09 & 0.04 & -0.01 & -0.22 \\
 & UnifDiff & -0.34* & -- & -- & 0.04 & -0.17 & -- & -- & -0.02 & -- & -- & 0.17 & \textbf{0.39**} & -- & -- \\
\midrule
Diffusion-based & Multi-Scale CKA
& 0.28 & 0.06 & 0.32 & -0.14 & -0.02 & 0.10 & -0.04
& 0.42 & 0.02 & 0.28 & -0.01 & -0.26 & -0.04 & 0.02 \\
(Ours) & AD-CKA
& 0.33 & -0.02 & 0.34 & -0.09 & -0.06 & 0.08 & -0.02
& \textbf{0.54} & 0.04 & 0.32 & 0.12 & 0.10 & -0.12 & -0.03 \\
& Multi-Scale RBF-CKA
& 0.17 & -0.09 & 0.24 & -0.05 & 0.18 & 0.15 & -0.22
& 0.32 & -0.25 & \textbf{0.37} & -0.01 & -0.21 & -0.00 & -0.16 \\
& AD-RBF-CKA
& 0.06 & 0.24 & 0.13 & 0.01 & 0.03 & 0.05 & -0.16
& 0.21 & 0.19 & 0.14 & -0.08 & -0.19 & -0.08 & -0.09 \\
& Multi-Scale DistCorr
& 0.25 & 0.01 & 0.30 & -0.09 & 0.05 & 0.11 & -0.08
& 0.40 & -0.08 & 0.30 & 0.00 & -0.26 & -0.04 & -0.01 \\
& AD-DistCorr
& 0.21 & 0.14 & 0.29 & -0.09 & 0.03 & 0.08 & -0.06
& 0.44 & 0.08 & 0.29 & 0.01 & 0.12 & -0.13 & -0.04 \\
\bottomrule
\end{tabular}%
}
\end{table*}

\begin{table*}[t]
\centering
\scriptsize
\setlength{\tabcolsep}{3.2pt}
\renewcommand{\arraystretch}{1.05}
\caption{Results of Test 1 (Correlation to Accuracy Difference) and Test 2 (Correlation to Output Difference) for the language domain.}
\label{tab:appendix-table13-extended}
\resizebox{\textwidth}{!}{%
\begin{tabular}{llrrrrrrrrr@{\hspace{1.0em}}rrrrrrrrr}
\toprule
\multicolumn{2}{c}{Dataset} & \multicolumn{9}{c}{MNLI} & \multicolumn{9}{c}{SST2} \\
\multicolumn{2}{c}{Test}
& \multicolumn{3}{c}{Acc.\ Corr.\ $\uparrow$}
& \multicolumn{3}{c}{JSD Corr.\ $\uparrow$}
& \multicolumn{3}{c}{Disagr.\ Corr.\ $\uparrow$}
& \multicolumn{3}{c}{Acc.\ Corr.\ $\uparrow$}
& \multicolumn{3}{c}{JSD Corr.\ $\uparrow$}
& \multicolumn{3}{c}{Disagr.\ Corr.\ $\uparrow$} \\
\multicolumn{2}{c}{Architecture}
& BERT & ALBERT & SmolLM2
& BERT & ALBERT & SmolLM2
& BERT & ALBERT & SmolLM2
& BERT & ALBERT & SmolLM2
& BERT & ALBERT & SmolLM2
& BERT & ALBERT & SmolLM2 \\
\midrule
CCA & PWCCA
& 0.01 & -0.03 & 0.04
& 0.22 & -0.17 & 0.53**
& -0.37* & -0.36** & 0.32**
& -0.33* & -- & --
& -0.32** & -- & --
& -0.22** & -- & -- \\
& SVCCA
& \textbf{0.32*} & -0.00 & 0.17
& \textbf{0.47**} & 0.12 & 0.06
& 0.00 & 0.33* & 0.08
& -0.08 & \textbf{0.66**} & 0.40**
& 0.47** & 0.35* & 0.49**
& 0.49** & 0.58** & 0.47** \\
Alignment & AlignCos
& 0.25 & 0.00 & 0.19
& 0.37* & 0.09 & \textbf{0.66**}
& -0.16 & -0.03 & 0.49**
& 0.02 & 0.13 & 0.44**
& 0.49** & 0.77** & 0.39**
& 0.28** & 0.51** & 0.37* \\
& AngShape
& 0.28 & 0.12 & 0.14
& 0.26 & -0.08 & 0.39**
& -0.02 & -0.01 & 0.23
& -0.14 & 0.34* & 0.46**
& 0.40** & 0.40** & 0.36*
& 0.48** & 0.51** & 0.24 \\
& HardCorr
& 0.04 & 0.21 & 0.01
& -0.27 & -0.03 & 0.50**
& -0.43** & 0.00 & 0.30*
& -0.33* & 0.27 & 0.53**
& -0.01 & 0.40** & 0.33*
& 0.34** & 0.56** & 0.23 \\
& LinReg
& 0.18 & 0.05 & -0.01
& 0.28** & 0.01 & 0.62**
& -0.03 & -0.13 & 0.33**
& -0.38** & 0.23* & 0.03
& 0.02 & 0.37** & 0.34**
& 0.32** & 0.37** & 0.31** \\
& OrthProc
& 0.28 & 0.12 & 0.14
& 0.26 & -0.08 & 0.39**
& -0.02 & -0.01 & 0.23
& -0.14 & 0.34* & 0.46**
& 0.40** & 0.40** & 0.36*
& 0.48** & 0.51** & 0.24 \\
& PermProc
& 0.09 & -0.02 & 0.15
& -0.06 & 0.04 & 0.45**
& -0.30* & -0.04 & 0.31*
& -0.09 & 0.01 & 0.47**
& 0.04 & \textbf{0.79**} & 0.16
& 0.18* & 0.44** & 0.42** \\
& ProcDist
& 0.28 & -0.04 & 0.17
& 0.07 & -0.00 & 0.60**
& -0.38* & -0.07 & \textbf{0.43**}
& 0.10 & 0.03 & 0.57**
& 0.49** & 0.76** & 0.36*
& 0.38** & 0.47** & 0.52** \\
& SoftCorr
& 0.11 & 0.18 & 0.07
& -0.23 & -0.02 & 0.61**
& -0.42** & 0.01 & 0.37*
& -0.33* & 0.29 & 0.55**
& -0.01 & 0.37* & 0.37*
& 0.36** & 0.56** & 0.31* \\
RSM & CKA
& 0.18 & 0.17 & 0.22
& 0.30* & 0.28 & 0.08
& -0.01 & \textbf{0.57**} & 0.12
& -0.06 & \textbf{0.66**} & 0.48**
& 0.48** & 0.37* & 0.52**
& 0.51** & 0.58** & 0.50** \\
& RBF-CKA
& 0.07 & 0.03 & 0.09
& 0.27 & -0.01 & 0.36
& 0.54 & 0.13 & 0.29
& -0.16 & 0.58 & 0.44
& 0.11 & 0.07 & 0.60
& 0.12 & 0.29 & 0.44 \\
& DistCorr
& 0.15 & \textbf{0.25} & 0.18
& 0.39** & \textbf{0.32*} & 0.13
& 0.12 & \textbf{0.57**} & 0.17
& -0.10 & 0.56** & 0.49**
& 0.51** & 0.51** & 0.57**
& 0.53** & 0.58** & 0.51** \\
& EOS
& 0.03 & -0.10 & -0.11
& 0.36* & -0.10 & -0.30
& 0.01 & -0.16 & -0.23
& -0.38** & -0.06 & -0.27
& -0.21** & 0.13 & 0.06
& -0.07 & -0.05 & 0.06 \\
& GULP
& -0.01 & -0.12 & 0.01
& 0.35* & -0.10 & 0.25
& 0.02 & -0.16 & 0.13
& -0.36* & -0.17 & --
& -0.30** & 0.25 & --
& -0.12 & -0.13 & -- \\
& RSA
& 0.00 & 0.18 & \textbf{0.24}
& 0.27 & 0.23 & -0.01
& 0.19 & 0.47** & 0.00
& -0.23 & 0.43** & 0.39**
& 0.44** & 0.53** & 0.26
& \textbf{0.59**} & 0.58** & 0.15 \\
& RSMDiff
& 0.30* & -0.15 & 0.06
& -0.18 & -0.02 & 0.35*
& -0.19 & 0.12 & 0.24
& 0.20 & -0.05 & 0.36*
& 0.24** & 0.43** & -0.03
& -0.10 & 0.28 & 0.15 \\
Neighbors & 2nd-Cos
& -0.26 & -0.25 & 0.03
& 0.16 & 0.04 & -0.07
& \textbf{0.55**} & 0.12 & 0.01
& 0.30* & 0.24 & 0.10
& \textbf{0.49**} & 0.30* & 0.18
& 0.15* & 0.16 & -0.20 \\
& Jaccard
& -0.21 & -0.25 & -0.03
& 0.13 & 0.02 & 0.09
& 0.17 & -0.05 & -0.03
& 0.17 & 0.23 & 0.31*
& \textbf{0.54**} & 0.32* & 0.24
& 0.23** & 0.12 & 0.06 \\
& RankSim
& -0.09 & -0.27 & -0.05
& 0.08 & -0.06 & 0.08
& 0.05 & -0.13 & -0.01
& 0.14 & 0.12 & 0.18
& 0.56** & 0.26 & 0.15
& 0.24** & 0.09 & 0.01 \\
Topology & IMD
& -0.26 & -0.08 & -0.13
& -0.39** & -0.29* & 0.12
& -0.06 & -0.16 & 0.09
& -0.06 & 0.02 & 0.31*
& 0.02 & 0.22 & 0.23
& 0.21** & 0.08 & 0.25 \\
& RTD
& 0.09 & -0.27 & -0.02
& 0.04 & -0.34* & 0.33*
& 0.10 & -0.18 & 0.12
& 0.05 & -0.10 & 0.32*
& 0.33** & -0.08 & 0.27
& 0.13 & 0.03 & 0.06 \\
Statistic & ConcDiff
& -0.00 & -0.07 & 0.01
& 0.02 & -0.07 & 0.46**
& -0.31* & 0.07 & 0.30*
& 0.12 & -0.20 & 0.06
& -0.02 & 0.59** & -0.11
& -0.26** & 0.29 & 0.06 \\
& MagDiff
& 0.22 & -0.06 & -0.20
& 0.01 & 0.01 & -0.39**
& -0.03 & 0.08 & -0.46**
& -0.13 & 0.01 & -0.10
& 0.11 & 0.10 & -0.07
& 0.04 & 0.13 & -0.17 \\
& UnifDiff
& 0.14 & -0.16 & -0.12
& -0.02 & -0.30* & -0.05
& -0.14 & -0.24 & 0.03
& 0.35* & 0.27 & -0.13
& 0.38** & 0.20 & 0.05
& -0.01 & 0.17 & -0.01 \\
\midrule
Diffusion-based & Multi-Scale CKA
& 0.27 & 0.14 & 0.22
& 0.33 & 0.26 & 0.07
& -0.08 & 0.53 & 0.11
& -0.06 & \textbf{0.66} & 0.46
& 0.39 & 0.37 & 0.51
& 0.42 & 0.58 & 0.49 \\
(Ours) & AD-CKA
& -0.09 & -0.25 & 0.02
& -0.02 & 0.24 & -0.31
& -0.06 & 0.34 & -0.21
& \textbf{0.43} & 0.14 & \textbf{0.58}
& 0.40 & 0.12 & 0.66
& 0.39 & 0.27 & 0.45 \\
& Multi-Scale RBF-CKA
& 0.02 & 0.12 & 0.10
& 0.20 & 0.04 & 0.26
& 0.43 & 0.17 & 0.19
& -0.17 & 0.55 & 0.47
& 0.13 & 0.03 & 0.59
& 0.16 & 0.25 & 0.48 \\
& AD-RBF-CKA
& -0.14 & -0.15 & 0.13
& 0.05 & 0.13 & 0.24
& -0.05 & 0.30 & 0.24
& 0.24 & 0.19 & 0.53
& 0.28 & 0.04 & 0.47
& 0.37 & 0.25 & 0.39 \\
& Multi-Scale DistCorr
& 0.22 & \textbf{0.25} & 0.16
& 0.40 & 0.27 & 0.06
& -0.03 & 0.56 & 0.12
& -0.12 & 0.58 & 0.49
& 0.41 & \textbf{0.51} & 0.57
& 0.44 & \textbf{0.60} & 0.53 \\
& AD-DistCorr
& -0.11 & -0.23 & 0.01
& 0.03 & 0.19 & -0.34
& -0.03 & 0.35 & -0.19
& 0.25 & 0.18 & 0.55
& 0.38 & 0.11 & \textbf{0.74}
& 0.41 & 0.31 & \textbf{0.54} \\
\bottomrule
\end{tabular}%
}
\end{table*}

\begin{table*}[t]
\centering
\scriptsize
\setlength{\tabcolsep}{3.5pt}
\renewcommand{\arraystretch}{1.05}
\caption{Results of Test 1 (Correlation to Accuracy Difference) for the graph domain}
\label{tab:appendix-table7-extended}
\resizebox{\textwidth}{!}{%
\begin{tabular}{llrrrrrrrrrr}
\toprule
\multicolumn{2}{c}{Dataset} & \multicolumn{4}{c}{Cora} & \multicolumn{3}{c}{Flickr} & \multicolumn{3}{c}{OGBN-Arxiv} \\
\multicolumn{2}{c}{Arch.} & GCN & SAGE & GAT & PGNN & GCN & SAGE & GAT & GCN & SAGE & GAT \\
\midrule
CCA & PWCCA & -0.05 & 0.07 & -0.26* & -0.09 & -- & -0.05 & -0.16 & -0.09 & 0.06 & -0.30** \\
 & SVCCA & -0.02 & -0.13 & -0.19 & -0.33* & 0.01 & 0.01 & -0.18 & -0.27 & 0.10 & -0.10 \\
Alignment & AlignCos & -0.33* & 0.13 & -0.29 & -0.08 & 0.35* & 0.24 & -0.07 & -0.08 & 0.17 & -0.17 \\
 & AngShape & 0.15 & -0.29 & -0.12 & -0.02 & 0.39** & 0.28 & -0.15 & -0.04 & 0.09 & -0.09 \\
 & HardCorr & 0.11 & -0.11 & -0.14 & -0.12 & 0.31* & 0.35* & 0.06 & \textbf{0.37*} & 0.02 & 0.04 \\
 & LinReg & 0.06 & -0.21* & -0.11 & -0.14 & -0.04 & 0.17 & -0.18 & 0.07 & -0.01 & -0.19 \\
 & OrthProc & 0.15 & -0.29 & -0.12 & -0.02 & 0.39** & 0.28 & -0.15 & -0.04 & 0.09 & -0.09 \\
 & PermProc & -0.12 & \textbf{0.18} & -0.26 & \textbf{0.29} & 0.20 & -0.19 & 0.15 & -0.09 & 0.03 & \textbf{0.43**} \\
 & ProcDist & 0.08 & 0.01 & -0.19 & \textbf{0.29} & 0.02 & -0.06 & 0.11 & -0.17 & 0.07 & \textbf{0.43**} \\
 & SoftCorr & \textbf{0.18} & -0.05 & 0.03 & -0.02 & 0.30* & 0.33* & -0.07 & 0.35* & 0.12 & 0.12 \\
RSM & CKA & 0.16 & -0.17 & 0.02 & -0.08 & 0.03 & 0.27 & -0.16 & -0.17 & 0.11 & -0.05 \\
& RBF-CKA & 0.09 & -0.23 & 0.02 & -- & 0.52 & 0.10 & -0.15 & -0.14 & 0.19 & -0.13 \\
 & DistCorr & 0.01 & -0.17 & 0.03 & 0.17 & 0.41** & \textbf{0.42**} & -0.19 & -0.10 & 0.15 & -0.06 \\
 & EOS & -0.24 & 0.08 & -0.05 & -0.04 & 0.15 & -0.27 & \textbf{0.29} & -0.21 & 0.05 & -0.32* \\
 & GULP & -0.43** & 0.08 & -0.12 & -0.02 & 0.04 & -0.27 & -0.27 & -0.08 & 0.06 & -0.34* \\
 & RSA & -0.27 & 0.04 & -0.28 & 0.21 & 0.53** & 0.32* & -0.08 & -0.07 & \textbf{0.25} & 0.32* \\
 & RSMDiff & -0.19 & 0.07 & -0.14 & 0.24 & -0.18 & -0.16 & 0.13 & -0.05 & -0.19 & 0.02 \\
Neighbors & 2nd-Cos & -0.10 & 0.04 & -0.11 & 0.13 & \textbf{0.54**} & -0.19 & 0.01 & -0.47** & 0.22 & -0.19 \\
 & Jaccard & 0.02 & -0.11 & -0.16 & 0.07 & 0.32* & 0.28 & -0.18 & -0.32* & -0.13 & -0.14 \\
 & RankSim & -0.00 & -0.10 & 0.32* & 0.23 & 0.35* & 0.31* & -0.20 & -0.28 & 0.05 & -0.09 \\
Topology & IMD & -0.10 & 0.00 & -0.03 & 0.04 & -0.10 & 0.37* & -0.09 & -0.21 & -0.02 & -0.15 \\
 & RTD & 0.06 & -0.26 & \textbf{0.13} & -0.44** & 0.19 & 0.13 & -0.17 & -0.23 & 0.15 & -0.33* \\
Statistic & ConcDiff & 0.15 & -0.25 & -0.20 & 0.13 & -0.08 & -0.29 & -0.07 & -0.07 & -0.13 & -0.12 \\
 & MagDiff & 0.08 & -0.13 & -0.19 & 0.18 & 0.02 & -0.17 & 0.14 & -0.18 & -0.20 & 0.11 \\
 & UnifDiff & -0.10 & -0.04 & -0.09 & -0.13 & -0.18 & 0.03 & -0.18 & -0.19 & -0.20 & -0.25 \\
\midrule
Diffusion-based & Multi-Scale CKA & 0.09 & -0.20 & -0.02 & -- & 0.00 & 0.20 & -0.17 & -0.20 & 0.10 & -0.04 \\
(Ours) & AD-CKA & 0.06 & -0.10 & 0.03 & -- & 0.03 & 0.01 & -0.12 & -0.16 & -0.00 & -0.08 \\
 & Multi-Scale RBF-CKA & 0.07 & -0.29 & 0.02 & -- & 0.52 & 0.03 & -0.15 & -0.09 & 0.15 & -0.15 \\
 & AD-RBF-CKA & 0.08 & -0.25 & 0.02 & -- & 0.42 & -0.04 & -0.19 & -0.02 & -0.04 & -0.06 \\
 & Multi-Scale DistCorr & 0.06 & -0.20 & -0.01 & -- & 0.32 & 0.36 & -0.21 & -0.09 & 0.14 & -0.00 \\
 & AD-DistCorr & 0.04 & -0.18 & 0.03 & -- & 0.51 & 0.14 & -0.14 & -0.05 & -0.07 & 0.01 \\
\bottomrule
\end{tabular}%
}
\end{table*}

\begin{table*}[t]
\centering
\scriptsize
\setlength{\tabcolsep}{3.4pt}
\renewcommand{\arraystretch}{1.05}
\caption{Exemplary results for selected datasets and models. We show results of GraphSAGE on Flickr for graphs, BERT on SST2 for language, and ResNet18 on ImageNet100 for vision. For the selected tests, higher values indicate better adherence to the corresponding similarity grounding. For details regarding tests other than accuracy and JSD correlation, refer~\cite{klabunde2025resi}}
\label{tab:appendix-table3-extended}
\resizebox{\textwidth}{!}{%
\begin{tabular}{llrrrr@{\hspace{1.2em}}rrrr@{\hspace{1.2em}}rrrr}
\toprule
\multicolumn{2}{c}{Domain} & \multicolumn{4}{c}{Graphs} & \multicolumn{4}{c}{Language} & \multicolumn{4}{c}{Vision} \\
\multicolumn{2}{c}{Dataset} & \multicolumn{4}{c}{Flickr} & \multicolumn{4}{c}{SST-2} & \multicolumn{4}{c}{ImageNet100} \\
\multicolumn{2}{c}{Architecture} & \multicolumn{4}{c}{GraphSAGE} & \multicolumn{4}{c}{BERT} & \multicolumn{4}{c}{ResNet18} \\
\multicolumn{2}{c}{Test}
& Acc.\ Corr. & Label Rand. & Shortcut & Aug.
& Acc.\ Corr. & JSD Corr. & Label Rand. & Shortcut
& Acc.\ Corr. & JSD Corr. & Shortcut & Aug. \\
\midrule
CCA & PWCCA
& -0.05 & 0.44 & 0.43 & 0.57
& -0.33 & -0.32 & 0.27 & 0.32
& -0.02 & 0.13 & 0.99 & 0.90 \\
& SVCCA
& 0.01 & 0.80 & 0.93 & 0.67
& -0.08 & 0.47 & 0.64 & 0.36
& 0.29 & 0.21 & 0.55 & 0.40 \\
Alignment & AlignCos
& 0.24 & 0.42 & \textbf{1.00} & 0.70
& 0.02 & 0.49 & \textbf{0.99} & 0.54
& -0.08 & 0.08 & \textbf{1.00} & 0.71 \\
& AngShape
& 0.28 & 0.43 & \textbf{1.00} & 0.76
& -0.14 & 0.40 & 0.49 & 0.43
& 0.21 & 0.24 & \textbf{1.00} & 0.71 \\
& HardCorr
& 0.35 & 0.46 & \textbf{1.00} & 0.72
& -0.33 & -0.01 & 0.44 & 0.36
& 0.21 & 0.28 & 0.97 & 0.46 \\
& LinReg
& 0.17 & 0.45 & 0.61 & 0.81
& -0.38 & 0.02 & 0.40 & 0.46
& 0.19 & 0.21 & 0.99 & 0.94 \\
& OrthProc
& 0.28 & 0.43 & \textbf{1.00} & 0.76
& -0.14 & 0.40 & 0.49 & 0.43
& 0.21 & 0.24 & \textbf{1.00} & 0.71 \\
& PermProc
& -0.19 & 0.90 & \textbf{1.00} & 0.69
& -0.09 & 0.04 & 0.45 & 0.55
& 0.07 & 0.18 & 0.72 & 0.41 \\
& ProcDist
& -0.06 & 0.62 & \textbf{1.00} & 0.81
& 0.10 & 0.49 & 0.86 & 0.52
& 0.08 & 0.10 & \textbf{1.00} & 0.58 \\
& SoftCorr
& 0.33 & 0.45 & \textbf{1.00} & 0.58
& -0.33 & -0.01 & 0.48 & 0.34
& 0.27 & \textbf{0.45} & 0.97 & 0.45 \\
RSM & CKA
& 0.27 & 0.66 & \textbf{1.00} & 0.75
& -0.06 & 0.48 & 0.59 & 0.38
& 0.36 & 0.30 & \textbf{1.00} & 0.90 \\
& RBF-CKA
& 0.10 & 0.46 & \textbf{1.00} & 0.96
& -0.16 & 0.11 & 0.53 & \textbf{1.00}
& 0.25 & 0.17 & \textbf{1.00} & 0.65 \\
& DistCorr
& \textbf{0.42} & 0.43 & \textbf{1.00} & 0.79
& -0.10 & 0.51 & 0.59 & 0.39
& 0.31 & 0.26 & \textbf{1.00} & 0.83 \\
& EOS
& -0.27 & 0.42 & 0.43 & 0.53
& -0.38 & -0.21 & 0.36 & 0.33
& 0.05 & 0.09 & \textbf{1.00} & 0.93 \\
& GULP
& -0.27 & 0.42 & 0.43 & 0.54
& -0.36 & -0.30 & 0.28 & 0.30
& 0.02 & 0.07 & \textbf{1.00} & 0.92 \\
& RSA
& 0.32 & 0.42 & \textbf{1.00} & \textbf{0.98}
& -0.23 & 0.44 & 0.48 & 0.47
& 0.06 & 0.12 & \textbf{1.00} & \textbf{0.98} \\
& RSMDiff
& -0.16 & 0.92 & 0.92 & 0.93
& 0.20 & 0.24 & 0.91 & 0.37
& 0.09 & -0.41 & 0.57 & 0.45 \\
Neighbors & 2nd-Cos
& -0.19 & 0.42 & \textbf{1.00} & 0.92
& 0.30 & 0.49 & 0.37 & \textbf{0.64}
& -0.08 & -0.13 & \textbf{1.00} & 0.78 \\
& Jaccard
& 0.28 & 0.43 & 0.83 & 0.88
& 0.17 & 0.54 & 0.35 & \textbf{0.64}
& -0.11 & 0.36 & \textbf{1.00} & 0.79 \\
& RankSim
& 0.31 & 0.43 & 0.77 & 0.88
& 0.14 & \textbf{0.56} & 0.34 & \textbf{0.64}
& 0.07 & -0.15 & 0.99 & 0.71 \\
Topology & IMD
& 0.37 & 0.37 & 0.97 & 0.57
& -0.06 & 0.02 & 0.47 & 0.34
& 0.17 & -0.11 & 0.67 & 0.56 \\
& RTD
& 0.13 & 0.59 & \textbf{1.00} & \textbf{1.00}
& 0.05 & 0.33 & 0.27 & 0.39
& 0.09 & -0.18 & \textbf{1.00} & 0.57 \\
Statistic & ConcDiff
& -0.29 & 0.57 & 0.18 & 0.35
& 0.12 & -0.02 & 0.96 & 0.27
& -0.11 & -0.29 & 0.53 & 0.43 \\
& MagDiff
& -0.17 & 0.72 & 0.78 & 0.17
& -0.13 & 0.11 & 0.38 & 0.35
& -0.16 & -0.38 & 0.37 & 0.37 \\
& UnifDiff
& 0.03 & 0.90 & 0.50 & 0.24
& 0.35 & 0.38 & 0.60 & 0.37
& -0.18 & -0.34 & 0.75 & 0.17 \\
\midrule
Diffusion-based & Multi-Scale CKA
& 0.20 & 0.74 & \textbf{1.00} & 0.76
& -0.06 & 0.39 & 0.61 & \textbf{1.00}
& \textbf{0.38} & 0.28 & \textbf{1.00} & 0.74 \\
(Ours) & AD-CKA
& 0.01 & \textbf{0.95} & 0.96 & 0.75
& \textbf{0.43} & 0.40 & 0.62 & 0.99
& 0.17 & 0.33 & 0.91 & 0.66 \\
& Multi-Scale RBF-CKA
& 0.03 & 0.51 & \textbf{1.00} & 0.96
& -0.17 & 0.13 & 0.53 & 0.99
& 0.25 & 0.17 & 0.97 & 0.65 \\
& AD-RBF-CKA
& -0.04 & 0.65 & \textbf{1.00} & \textbf{1.00}
& 0.24 & 0.28 & 0.58 & 0.96
& 0.14 & 0.06 & 0.96 & 0.64 \\
& Multi-Scale DistCorr
& 0.36 & 0.43 & \textbf{1.00} & 0.78
& -0.12 & 0.41 & 0.60 & \textbf{1.00}
& 0.31 & 0.25 & 0.99 & 0.60 \\
& AD-DistCorr
& 0.14 & 0.72 & \textbf{1.00} & 0.77
& 0.25 & 0.38 & 0.62 & \textbf{1.00}
& 0.17 & 0.21 & 0.92 & 0.64 \\
\bottomrule
\end{tabular}%
}
\end{table*}

\section{Further details on the GRS benchmark}
\label{app:grs}

\paragraph{Benchmark 4: Challenge set.}
Following \cite{ding2021grounding}, Benchmark~4 evaluates whether a representation
dissimilarity measure tracks out-of-distribution (OOD) accuracy under multiple
sources of training randomness. Let $\mathcal{F}$ denote the set of $100$
BERT-medium models obtained from all combinations of $10$ pretraining seeds and
$10$ fine-tuning seeds, and let $\mathbf{R}^{(\ell)}_f$ denote the
representation of model $f \in \mathcal{F}$ at layer $\ell \in \{1,\dots,8\}$.
For a fixed OOD task, let $a_{\mathrm{OOD}}(f)$ denote the corresponding OOD accuracy of model $f$. For each layer $\ell$, the reference model is chosen as
the one with highest OOD accuracy,
\begin{align}
f_\ell^\star = \arg\max_{f \in \mathcal{F}} a_{\mathrm{OOD}}(f),
\end{align}
and every other model is compared to this reference. The benchmark score at
layer $\ell$ is then the rank correlation between representation dissimilarity
and OOD accuracy difference:
\begin{align}
\mathrm{corr}\!\left(
\left\{ d\!\left(\mathbf{R}^{(\ell)}_{f_\ell^\star}, \mathbf{R}^{(\ell)}_f\right) \right\}_{f \in \mathcal{F}\setminus\{f_\ell^\star\}},
\left\{ \left| a_{\mathrm{OOD}}(f_\ell^\star) - a_{\mathrm{OOD}}(f) \right| \right\}_{f \in \mathcal{F}\setminus\{f_\ell^\star\}}
\right).
\end{align}
Here, $d\!\left(\mathbf{R}^{(\ell)}_{f_\ell^\star}, \mathbf{R}^{(\ell)}_f\right) = 1-m\!\left(\mathbf{R}^{(\ell)}_{f_\ell^\star}, \mathbf{R}^{(\ell)}_f\right)$, where the definition of $m(\bm{R}_1, \bm{R}_2)$ from Section~\ref{subsec:rep_sim_m} is used. \cite{ding2021grounding} reports the resulting rank correlations averaged over layers, and notes that Benchmark~4 is a challenging setting in which no measure achieves a high correlation.

\paragraph{Integration of our multi-layer measures.} In the original GRS setup, the functionality of interest is defined at the level of the full model, but representational similarity is computed separately at each layer and the resulting rank correlations are then averaged across layers. For our multi-layer measures, this admits a natural network-level adaptation: rather than comparing layers independently, we compute the similarities according to Definitions~\ref{def:adcka} and \ref{def:addistcorr}.

\paragraph{Number of representations selected.} All 8 representations from the BERT model are selected for the set $\mathcal{A}$.

The stability analysis in Appendix~\ref{app:resi-stability} also applies here.

\end{document}